\pdfoutput=1
\documentclass[11pt, letterpaper]{article}
\usepackage{blindtext}
\usepackage{hyperref}
\hypersetup{
	colorlinks=true,
	linkcolor=blue!70!black,
	citecolor=blue!70!black,
	urlcolor=blue!70!black
}

\usepackage[english]{babel}
\usepackage[T1]{fontenc}
\usepackage{thm-restate}

\usepackage[margin=1in]{geometry}

\usepackage{amsmath}
\usepackage{amssymb}
\usepackage{amsthm}
\usepackage{booktabs}
\usepackage{algorithm}
\usepackage[lined,boxed,ruled,norelsize,algo2e,linesnumbered]{algorithm2e}
\usepackage{bbold}
\usepackage{bbm}

\usepackage{graphicx}
\graphicspath{{./figs/}}

\usepackage{cleveref}
\newtheorem{theorem}{Theorem}
\newtheorem{lemma}{Lemma}
\newtheorem{fact}{Fact}
\newtheorem{definition}{Definition}
\newtheorem{corollary}{Corollary}
\newtheorem{proposition}{Proposition}
\newtheorem{claim}{Claim}

\newtheorem{remark}{Remark}

\numberwithin{equation}{section}

\newcommand{\defeq}{:=}
\newcommand{\norm}[1]{\left\lVert#1\right\rVert}
\newcommand{\norms}[1]{\lVert#1\rVert}
\newcommand{\normf}[1]{\left\lVert#1\right\rVert_{\textup{F}}}
\newcommand{\normsf}[1]{\lVert#1\rVert_{\textup{F}}}
\newcommand{\normop}[1]{\left\lVert#1\right\rVert_{\textup{op}}}
\newcommand{\normsop}[1]{\lVert#1\rVert_{\textup{op}}}
\newcommand{\normtr}[1]{\left\lVert#1\right\rVert_{\textup{tr}}}

\newcommand{\inprod}[2]{\left\langle#1, #2\right\rangle}

\newcommand{\eps}{\epsilon}
\newcommand{\lam}{\lambda}

\newcommand{\R}{\mathbb{R}}

\newcommand{\N}{\mathbb{N}}

\newcommand{\half}{\frac{1}{2}}

\newcommand{\1}{\mathbbm{1}}
\newcommand{\0}{\mathbb{0}}
\newcommand{\E}{\mathbb{E}}

\newcommand{\Nor}{\mathcal{N}}

\newcommand{\Tr}{\textup{Tr}}

\newcommand{\ma}{\mathbf{A}}

\newcommand{\id}{\mathbf{I}}
\newcommand{\mI}{\id}
\usepackage{xcolor}
\definecolor{burntorange}{rgb}{0.8, 0.33, 0.0}

\newcommand{\tO}{\widetilde{O}}
\newcommand{\nnz}{\textup{nnz}}

\newcommand{\Par}[1]{\left(#1\right)}
\newcommand{\Brack}[1]{\left[#1\right]}
\newcommand{\Brace}[1]{\left\{#1\right\}}
\newcommand{\Abs}[1]{\left|#1\right|}

\newcommand{\oracle}{\mathcal{O}}

\newcommand{\poly}{\textup{poly}}

\newcommand{\mmu}{\mathbf{U}}
\newcommand{\mv}{\mathbf{V}}
\newcommand{\mx}{\mathbf{X}}
\newcommand{\my}{\mathbf{Y}}
\newcommand{\md}{\mathbf{D}}
\newcommand{\mb}{\mathbf{B}}
\newcommand{\mm}{\mathbf{M}}
\newcommand{\mn}{\mathbf{N}}
\newcommand{\me}{\mathbf{E}}

\newcommand{\mmus}{\mathbf{U}_\star}
\newcommand{\mvs}{\mathbf{V}_\star}

\newcommand{\wt}{\widetilde}

\newcommand{\tx}{\tilde{x}}

\newcommand{\msig}{\boldsymbol{\Sigma}}
\newcommand{\msigs}{\boldsymbol{\Sigma}_{\star}}

\newcommand{\wh}{\widehat}
\newcommand{\mk}{\mathbf{K}}

\newcommand{\proj}{\boldsymbol{\Pi}}
\newcommand{\mzero}{\mathbf{0}}
\newcommand{\mms}{\mm^\star}
\newcommand{\mmh}{\widehat{\mm}}

\newcommand{\hr}{\hat{r}}
\newcommand{\rs}{r^\star}
\newcommand{\tv}{\tilde{v}}
\newcommand{\mmp}{\mathbf{P}}
\newcommand{\hmp}{\widehat{\mmp}}
\newcommand{\hme}{\widehat{\me}}
\newcommand{\tmm}{\widetilde{\mm}}
\newcommand{\Filter}{\mathsf{Filter}}
\newcommand{\Sparsify}{\mathsf{Sparsify}}
\newcommand{\Descent}{\mathsf{Descent}}
\newcommand{\Power}{\mathsf{Power}}
\newcommand{\orzo}{\oracle_{[0, 1]}}
\newcommand{\mmc}{\mm^\circ}
\newcommand{\gdrop}{\gamma_{\textup{drop}}}
\newcommand{\gadd}{\gamma_{\textup{add}}}
\newcommand{\mds}{\md^\star}
\newcommand{\ptot}{p_{\textup{tot}}}
\newcommand{\hc}{\widehat{c}}

\newcommand{\tmv}{\mathcal{T}_{\textup{mv}}}
\newcommand{\codeInput}{\textbf{Input:} }

\newcommand{\Test}{\mathsf{Test}}
\newcommand{\Representative}{\mathsf{Representative}}
\newcommand{\Count}{\mathsf{count}}
\newcommand{\tDelta}{\widetilde{\Delta}}
\newcommand{\mw}{\mathbf{W}}
\newcommand{\mz}{\mathbf{Z}}
\newcommand{\LowRankDist}{\mathsf{LowRankDist}}
\newcommand{\Agg}{\mathsf{Aggregate}}
\newcommand{\tmd}{\widetilde{\md}}
\newcommand{\mq}{\mathbf{Q}}
\newcommand{\mg}{\mathbf{G}}

\newcommand{\Complete}{\mathsf{Complete}}
\newcommand{\AGD}{\mathsf{AGD}}
\newcommand{\hmv}{\widehat{\mv}}
\newcommand{\Fix}{\mathsf{Fix}}
\newcommand{\OpNorm}{\mathsf{EstimateOpNorm}}
\newcommand{\MC}{\mathsf{MatrixCompletion}}
\newcommand{\PMC}{\mathsf{PartialMatrixCompletion}}
\newcommand{\Cfix}{C_{\textup{fix}}}
\newcommand{\xs}{x^\star}

\title{Matrix Completion in Almost-Verification Time}
\author{
			Jonathan A.\ Kelner\thanks{MIT, {\tt kelner@mit.edu}.  Supported in part by NSF awards CCF-1955217,  CCF-1565235, and DMS-2022448.}
			\and
			Jerry Li\thanks{Microsoft Research, {\tt jerrl@microsoft.com}.}
			\and
			Allen Liu\thanks{MIT, {\tt cliu568@mit.edu}. This work was partially done while working as an intern at Microsoft Research, and was supported in part by an NSF Graduate Research Fellowship and a Fannie and John Hertz Foundation Fellowship.}
			\and
			Aaron Sidford\thanks{Stanford University, {\tt sidford@stanford.edu}. Supported in part by a Microsoft Research Faculty Fellowship, NSF CAREER Award CCF-1844855, NSF Grant CCF-1955039, a PayPal research award, and a Sloan Research Fellowship.}
			\and
			Kevin Tian\thanks{Microsoft Research, {\tt tiankevin@microsoft.com}.}
		}
\date{}

\begin{document}

\maketitle
\thispagestyle{empty} 
\begin{abstract}
We give a new framework for solving the fundamental problem of low-rank matrix completion, i.e., approximating a rank-$r$ matrix $\mathbf{M} \in \mathbb{R}^{m \times n}$ (where $m \ge n$) from random observations. First, we provide an algorithm which completes $\mathbf{M}$ on $99\%$ of rows and columns under no further assumptions on $\mathbf{M}$ from $\approx mr$ samples and using $\approx mr^2$ time. Then, assuming the row and column spans of $\mathbf{M}$ satisfy additional regularity properties, we show how to boost this partial completion guarantee to a full matrix completion algorithm by aggregating solutions to regression problems involving the observations.

In the well-studied setting where $\mathbf{M}$ has incoherent row and column spans, our algorithms complete $\mathbf{M}$ to high precision from $mr^{2+o(1)}$ observations in $mr^{3 + o(1)}$ time (omitting logarithmic factors in problem parameters), improving upon the prior state-of-the-art \cite{JainN15} which used $\approx mr^5$ samples and $\approx mr^7$ time. Under an assumption on the row and column spans of $\mathbf{M}$ we introduce (which is satisfied by random subspaces with high probability), our sample complexity improves to an almost information-theoretically optimal $mr^{1 + o(1)}$, and our runtime improves to $mr^{2 + o(1)}$. Our runtimes have the appealing property of matching the best known runtime to verify that a rank-$r$ decomposition $\mathbf{U}\mathbf{V}^\top$ agrees with the sampled observations.
We also provide robust variants of our algorithms that, given random observations from $\mathbf{M} + \mathbf{N}$ with $\|\mathbf{N}\|_{\textup{F}} \le \Delta$, complete $\mathbf{M}$ to Frobenius norm distance $\approx r^{1.5}\Delta$ in the same runtimes as the noiseless setting. Prior noisy matrix completion algorithms \cite{CandesP10} only guaranteed a distance of $\approx \sqrt{n}\Delta$.

\end{abstract}
\newpage

\pagenumbering{gobble}
\setcounter{tocdepth}{2}
{
\tableofcontents
}
\newpage
\pagenumbering{arabic}

\section{Introduction}\label{sec:intro}

Matrix completion is a fundamental and well-studied problem in both the theory and practice of computer science, machine learning, operations research, and statistics. Broadly, the matrix completion problem asks to recover a matrix $\mm \in \R^{m \times n}$ from a small (i.e., \ sublinear) number of randomly revealed, and potentially noisy, entries. This problem was originally studied in the context of collaborative filtering \cite{RennieS05} (see e.g., the Netflix challenge \cite{Netflix07}) and has since found a myriad of applications in diverse settings such as signal processing \cite{LinialLR95, SoY07}, genetics \cite{NatarajanD14}, social network analysis \cite{MahindreJGP19}, and traffic engineering \cite{GursunC12}. 

\paragraph{Structural assumptions.} In the absence of additional assumptions, matrix completion is impossible. Unless there is structure among the entries of $\mm$, then all of $\mm$ must be revealed for recovery (as otherwise unrevealed entries can be arbitrary). Correspondingly, there has been a long line of work developing algorithms for matrix completion under different structural assumptions on $\mm$. Perhaps the most prevalent and natural assumption placed on $\mm$ is that it is low-rank. This assumption is well-motivated for the matrices arising in collaborative filtering or signal processing, for example, as discussed in \cite{CandesR12}. Furthermore, rank-$r$ matrices $\mm \in \R^{m \times n}$ can be represented in $O((m+n)r)$-space simply by storing its rank-$r$ factorization. Consequently, na\"ive parameter-counting arguments suggest it may be possible to recover $\mm$ using $O((m+n)r)$ observations.

However, the assumption that $\mm$ is low-rank alone is insufficient to enable algorithms for matrix completion that use $o(mn)$ observations. If $\mm$ has a single non-zero entry, then it has rank-$1$, and yet $\Omega(mn)$ observations are required to recover the nonzero entry (and consequently $\mm$) with constant probability. Correspondingly, works on low-rank matrix completion place different additional structural assumptions that preclude such sparse obstacles to solving the problem.

The setting where $\mm$ has \emph{incoherent} row and column spans is particularly well-studied \cite{CandesR12}. A dimension-$r$ subspace of $\R^d$ is $\mu$-incoherent if no projection of a basis vector has squared norm more than $\frac{\mu r} d$, i.e., the subspace is well-spread over coordinates; we use ``incoherent subspace'' without a parameter to mean a $\tO(1)$-incoherent subspace.\footnote{Throughout $\widetilde{O}$ hides polylogarithmic factors in $m, n$, the inverse failure probability, and the relative accuracy.} Letting $\mmu \msig\mv^\top$ be a singular value decomposition (SVD) of $\mm$, and assuming $\mmu, \mv$ span incoherent subspaces (and an entrywise bound on $\mmu\mv^\top$), \cite{recht2011simpler} refined results of~\cite{candes2010power,CandesR12,keshavan2010matrix}, and demonstrated that there are polynomial-time algorithms completing $\mm$ from $\widetilde{O}((m + n) r)$ observations. 

The parameters used in the definition of incoherence are motivated  by the fact that they are satisfied with high probability by random rank-$r$ matrices. Consequently, prior work showed that matrix completion is information-theoretically possible so long as the structure of $\mm$ is ``suitably-random.''
However, it is perhaps unclear whether incoherence is the correct or best notion of ``suitably-random,'' aside from the post-hoc justification that it allows for efficient matrix completion.

\paragraph{Performance of matrix completion algorithms.} Despite a plethora of work on matrix completion when e.g., $\mm$ has incoherent row and column spans (discussed below and in greater detail in Section~\ref{ssec:related}), many surprisingly fundamental algorithmic questions remain unresolved. A number of key open problems relate to the runtime and robustness of existing matrix completion algorithms.

The aforementioned works of~\cite{candes2010power,CandesR12,recht2011simpler} developed polynomial-time algorithms for completing a rank-$r$ matrix $\mm \in \R^{m \times n}$ with $\tO(1)$-incoherent row and column spans from a near-optimal number of observations. These algorithms were based on semidefinite programming (SDP) for nuclear norm minimzation. The runtimes of state-of-the-art SDP solvers \cite{JiangKLP020, HuangJ0T022} have a substantial polynomial overhead over the number of observations, inhibiting their practical application. Motivated by this shortcoming, another line of work~\cite{keshavan2010matrix,hardt2014understanding,JainN15, yi2016fast} developed iterative first-order methods, based on alternating minimization or gradient descent, whose runtimes depend linearly on the dimension $\max(m, n)$.
However, the state-of-the-art algorithms with such runtime guarantees still incur fairly substantial overheads in problem parameters.
Prior to our work, the best runtime for incoherent low-rank matrix completion was by \cite{JainN15}, whose algorithm ran in time $\widetilde{O}((m + n) r^7)$.\footnote{A more recent work~\cite{cherapanamjeri2017nearly} claims an improved runtime over \cite{JainN15}. However, to obtain this result \cite{cherapanamjeri2017nearly} assumes a sublinear-time exact singular value decomposition subroutine (which does not currently exist), and it is unclear how to recover the runtime claim of the paper without such an assumption \cite{cherapanamjeri2022personal}.} A contemporaneous work of \cite{yi2016fast} yielded an incomparable runtime of $\tO((m + n)r^4 \kappa^5)$, where $\kappa$ is the multiplicative range of $\mm$'s singular values. 

Another parameterization of the performance of matrix completion algorithms, which is rife with open problems, is the degree to which they can handle noise in the observations. In the setting where $\mm$ is low-rank and has incoherent row and column spans, suppose that instead of observing random entries of $\mm$, the observations we see are of $\mm + \mn$ for a noise matrix $\mn$ satisfying $\normf{\mn} \le \Delta$. We are unaware of any information-theoretic barriers to recovering a matrix $\mmh$ satisfying $\normsf{\mmh - \mm} = O(\Delta)$ with no further assumptions. However, state-of-the-art polynomial-time algorithms are only able to achieve a Frobenius norm recovery guarantee of $O(\sqrt{\min(m, n)} \Delta)$, which loses a dimension-dependent factor. 
While other matrix completion algorithms in the literature also demonstrate robustness to noise, their guarantees either require additional assumptions on the noise such as sparsity, e.g., \cite{cherapanamjeri2017nearly,yi2016fast}, or break down for large $\Delta$, e.g., \cite{KeshavanMO09, GunasekarAGG13, hardt2014understanding, HardtW14}.

These open problems regarding the complexity of matrix completion give rise to the following key questions which motivate our work.
\begin{enumerate}
	\item What type of matrix completion is possible when the only structure is a rank bound? \label{item:assume}
    \item Are there alternative structural assumptions to incoherent subspaces which enable faster algorithms, improved sample complexities, and better noise tolerance?
    \label{item:alternative}
    \item Under the well-studied structural assumption of incoherent subspaces, to what extent can we improve upon the runtimes and error tolerance of existing matrix completion algorithms?\label{item:incoherence}
\end{enumerate}

\subsection{Our results}\label{ssec:results}

We provide a new algorithmic framework for matrix completion and technical tools that address the shortcomings raised by each of Questions~\ref{item:assume}, \ref{item:alternative}, and~\ref{item:incoherence}. 
The cornerstone of our framework is a new iterative method that answers  Question~\ref{item:assume} by obtaining (perhaps surprisingly) nontrivial matrix completion guarantees \emph{with no structural assumptions beyond a rank bound.}
We believe this result is of independent interest, and we state it first.

\paragraph{Partial matrix completion without structure.} As already noted, fully completing low-rank $\mm$ from partial observations is impossible without further assumptions due to the possibility of sparse, large entries. However, when $\mm$ is low-rank, such entries are necessarily rare (see Lemma~\ref{lem:cover-large-entries} for a formal statement) and thus one could still hope to recover a large portion of $\mm$. We demonstrate this in the following theorem (where $\mm_{S, T}$ denotes the submatrix indexed by $S \subseteq [m], T \subseteq [n]$).

\begin{theorem}[informal, see Corollary~\ref{cor:pmc}]
\label{thm:pmc-informal}
Let $m \ge n$,\footnote{All of our results handle $m \le n$ symmetrically via transposition, so we often assume $m \ge n$ for ease of exposition.} let $\mm \in \R^{m \times n}$ be rank-$r$, and let $\mn \in \R^{m \times n}$ satisfy $\normf{\mn} \le \Delta$. There is an algorithm which, given $\tO(m^{1 + o(1)} r)$ random observations from $\mm + \mn$, runs in time $\tO(m^{1+o(1)}r^2)$ and, with high probability, outputs a rank-$rm^{o(1)}$ factorization of $\mmh \in \R^{m \times n}$ so that there exist $S \subseteq [m]$ and $T \subseteq [n]$ with $|S| \ge 0.99m$, $|T| \ge 0.99n$, and 
\[\normf{\Brack{\mm - \mmh}_{S, T}} \le \Delta.\]
\end{theorem}
In other words, on a very large subset of coordinates, Theorem~\ref{thm:pmc-informal} recovers $\mm$ up to the optimal error threshold up to constants. Additionally, since $\approx mr$ samples are information-theoretically necessary to perform nontrivial (full) matrix completion 
\cite{candes2010power}, the sample complexity of Theorem~\ref{thm:pmc-informal} is almost-optimal. As a corollary,
 in the case when $\Delta = 0$, Theorem~\ref{thm:pmc-informal} shows that matrix completion can be solved exactly on all but $1\%$ of rows and columns (assuming a bounded bit complexity).

The runtime stated in Theorem~\ref{thm:pmc-informal} has the appealing property that it is what we call \emph{almost-verification time}. Consider the natural problem of verifying a rank-$r$ factorization of $\mm$, that is the problem of verifying that $\mmu \mv^\top = \mm$ on $mr$ observed entries given an explicit rank-$r$ factorization of $\mm = \mmu \mv^\top$, for  $\mmu \in \R^{m \times r}, \mv \in \R^{n \times r}$. The best known running time for this problem is $O(m r^2)$ (even when using fast multiplication). Up to subpolynomial factors, our runtime in Theorem~\ref{thm:pmc-informal} matches this natural bottleneck to improved runtimes for matrix completion. 


The guarantees of Theorem~\ref{thm:pmc-informal} are to the best of our knowledge new, and seem particularly striking in light of the long history of matrix completion algorithms. It is worth noting that there has been work which broadly aims to complete a submatrix from observations.
Perhaps the most closely-related result is due to recent, similarly-titled work of~\cite{kanade2022partial}, which studies a different notion of partial matrix completion.
\cite{kanade2022partial} shows that if $\mm$ is rank-$r$ and has bounded entries, and the distribution of observed entries is supported on a subset $U \subseteq [m] \times [n]$, then one can recover $\mm$ to constant average entrywise error on a subset of $[m] \times [n]$ with cardinality at least $|U|$ (for a suitable relaxed notion of average error). 
For instance, if the algorithm of \cite{kanade2022partial} is instantiated for $U = [m] \times [n]$, then it outputs $\widehat{\mm}$ satisfying $\normsf{\widehat{\mm} - \mm}^2 \leq \eps \normsf{\mm}^2$ using $O((m + n) r \eps^{-2})$ observations.
Notably, their complexity depends inverse-polynomially on the accuracy (and hence inhibits exact completion). 
In contrast, our Theorem~\ref{thm:pmc-informal} achieves exact completion (albeit only on a large submatrix), and works under the standard, i.i.d.\ observation model.

\paragraph{Matrix completion beyond incoherence.} Equipped with our new partial matrix completion subroutine, we turn to Question~\ref{item:alternative} and ask under what structural assumptions we can leverage it to solve (full) matrix completion efficiently. 
Given the generality of our partial matrix completion algorithm, it is natural to ask whether we can first run partial completion, and then recover the matrix on the small subset of rows and columns on which our partial completion method fails.

When analyzing this iterative process of recovering rows and columns of the target matrix which were dropped by our partial completion method, the standard structural assumption of incoherence turns out to be a lossy notion of ``suitably-random.'' 
Instead, we define a new structural assumption on subspaces which we call \emph{subspace regularity}, that serves as a proxy for randomness.

\begin{definition}[Regular subspace]\label{def:regular-subspace}
We say a subspace $V \subseteq \R^d$ is \emph{$(\alpha,\beta)$-regular} if for all $\alpha d$-sparse $v \in \R^d$, $\norm{\proj_{V_\perp} v}_2 \ge \beta \norm{v}_2$.
\end{definition}

Note that Definition~\ref{def:regular-subspace} implies $\norm{\proj_V v}_2 \le (1 - \beta^2) \norm{v}_2$, a condition which bears resemblance to incoherence (by bounding the relative weight of any small set of coordinates in the subspace). Intuitively, Definition~\ref{def:regular-subspace} imposes that the restriction of $V$ to a sufficiently large set of coordinates is still well-conditioned (made formal by Lemma~\ref{lem:regular-wc}). We prove that uniformly random subspaces are $(\alpha, \beta)$-regular for constant $\alpha$, $\beta$, with exponentially small failure probability, in Appendix~\ref{app:regular-subspace}. Subspace regularity is not directly comparable to incoherence without losing $r$ factors in the parameter settings (see Fact~\ref{fact:incoherent-standard}), because a $d \times r$ basis matrix for an incoherent subspace can be entirely supported on an $O(\frac 1 r)$ fraction of rows. However, Definition~\ref{def:regular-subspace} is naturally compatible with our partial matrix completion method: roughly speaking, we require that the non-dropped rows and columns (e.g., $(S, T)$ in Theorem~\ref{thm:pmc-informal}) are representative enough of the remaining matrix to recover dropped subsets. This representativeness is captured by the conditioning requirement in Definition~\ref{def:regular-subspace}. Our main (full) matrix completion result under subspace regularity is the following.

\begin{theorem}[informal, see Corollary~\ref{cor:main_regular}]\label{thm:informal_regular}
    Let $m\ge n$, let $\mm \in \R^{m \times n}$ be rank-$r$ and have $(\Omega(1), \Omega(1))$-regular row and column spans, and let $\mn \in \R^{m \times n}$ satisfy $\normf{\mn} \le \Delta$.
There is an algorithm which, given $\tO(mr^{1 + o(1)})$ random observations from $\mm + \mn$, runs in time $\tO(mr^{2 + o(1)})$ and with high probability outputs a rank-$r$ factorization of $\widehat{\mm} \in \R^{m \times n}$ so that 
\[
\normf{\widehat{\mm} - \mm} = O\Par{r^{1.5 + o(1)}  \cdot \Delta}.
\]
\end{theorem}

The sample complexity of Theorem~\ref{thm:informal_regular} is optimal up to subpolynomial factors \cite{candes2010power} in the noiseless case (captured by our result by taking $\Delta \to 0$); these subpolynomial factors arise due to iterate rank blowup issues discussed in Section~\ref{ssec:approach}. Moreover, the algorithm of Theorem~\ref{thm:informal_regular} runs in almost-verification time
for the number of observations.
Finally, in the noisy case, $\Delta > 0$, the overhead of Theorem~\ref{thm:informal_regular}'s recovery guarantee only scales with the rank $r$, as opposed to the prior state-of-the-art \cite{CandesR12} whose overhead scaled polynomially with the problem's dimensionality.

Even under subspace regularity, the ``fixing'' step used to obtain Theorem~\ref{thm:informal_regular} we briefly described is quite technically involved.
One of the main difficulties is that after running partial matrix completion, we do not necessarily know which rows and columns $S, T$ have been completed.
Our fixing algorithm circumvents this issue by carefully finding a small set of rows and columns which approximately span the row and column space of $\mm$ in a well-conditioned fashion, satisfying a ``representative'' condition we state in Definition~\ref{def:representative-subset}.
We then show that we can use these representative rows and columns, alongside held-out random observations of the matrix, to robustly recover the rows and columns that were incorrectly completed by the partial completion algorithm.
Putting these pieces together yields a fixing algorithm which recovers the subsets our partial completion method is inaccurate on, but increases error by a $\textup{poly}(r)$ factor. By carefully interleaving this fixing operation with repeated applications of our partial completion iterative method, we geometrically decrease the error of our overall algorithm. We give a detailed overview of our approach in Section~\ref{ssec:approach}.

\paragraph{Matrix completion with incoherence.} Finally, we return to Question~\ref{item:incoherence}, i.e., matrix completion under the well-studied assumption of incoherence. We demonstrate that a small modification of our algorithm in Theorem~\ref{thm:informal_regular} implies an analogous result under incoherence. In light of Fact~\ref{fact:incoherent-standard} (which converts a subspace incoherence bound into a regularity bound), this is immediate up to $\textup{poly}(r)$ losses in the sample complexity and runtime. We give a tighter characterization of the lossiness due to assuming incoherence by introducing Definition~\ref{def:standard-subspace}, which subsumes both subspace regularity and incoherence. Leveraging this characterization, our techniques imply the following result for incoherent matrix completion (losing a single $r$ factor in runtime and samples over Theorem~\ref{thm:informal_regular}).

\begin{corollary}[informal, see Corollary~\ref{cor:main_incoherent}]\label{cor:informal_incoherent}
    Let $m \ge n$, let $\mm \in \R^{m \times n}$ be rank-$r$ and have $\widetilde{O}(1)$-incoherent row and column spans, and let $\mn \in \R^{m \times n}$ satisfy $\normf{\mn} \le \Delta$.
There is an algorithm which, given $\tO(m r^{2 + o(1)})$ random observations from $\mm + \mn$, runs in time $\tO(mr^{3 + o(1)})$ and with high probability outputs a rank-$r$ factorization of $\mmh \in \R^{m \times n}$ so that
\[
\normf{\widehat{\mm} - \mm} = O\Par{r^{1.5 + o(1)} \cdot \Delta}.
\]
\end{corollary}
Even with this additional $r$ factor overhead, our results compare favorably to existing work on incoherent matrix completion.
As mentioned previously, the state-of-the-art runtime for incoherent matrix completion (with polylogarithmic dependence on problem conditioning) was $\widetilde{O}(m r^7)$ \cite{JainN15}, which our Corollary~\ref{cor:informal_incoherent} dramatically improves upon.
While the sample complexity of Corollary~\ref{cor:informal_incoherent} 
is a factor of $r$ larger than the sample complexity required by matrix completion algorithms based on semidefinite programming, all incoherent matrix completion methods in the literature which run in time nearly-linear in $m = \max(m, n)$ use $\Omega(m r^2)$ observations (and often more), which we match up to subpolynomial factors.
Additionally, none of the existing polynomial-time algorithms (even the slower semidefinite programming approaches!) were known to yield dimension-independent recovery guarantees for noisy incoherent matrix completion. We summarize how Corollary~\ref{cor:informal_incoherent} compares to prior work on matrix completion under incoherence below.

\begin{figure}[h!]
    \centering
    \label{fig:results}
    \begin{tabular}{ | c | c| c | c | } 
  \hline
  Algorithm & Sample complexity & Runtime & Recovery error \\ 
  \hline
  \cite{recht2011simpler} & $mr $ & $\Omega(m^\omega)$ & N/A \\
  \cite{CandesP10} & $mr$ & $\Omega(m^\omega)$ & $\sqrt n \Delta$  \\ 
  \cite{HardtW14} & $mr^9$ & $mr^{13}$ & $\star$ \\
  \cite{sun2016guaranteed} & $mr^7\kappa^4$ & $m^2r^6\kappa^4$ & N/A \\
  \cite{JainN15} & $mr^5$ & $mr^7$ & N/A \\
  \cite{yi2016fast} & $mr^2 \kappa^4$ & $mr^4\kappa^5$ & $\star$ \\
  Corollary~\ref{cor:informal_incoherent} & $mr^{2 + o(1)}$ & $mr^{3 + o(1)}$ & $r^{1.5 + o(1)} \Delta$ \\
  \hline
\end{tabular}
    \caption{Comparison of algorithms for completing rank-$r$ $\mm \in \R^{m \times n}$ with $\widetilde{O}(1)$-incoherent row and column spans, assuming $m \ge n$. We let $\Delta$ upper bound the (Frobenius norm) noise level, $\kappa$ denote the multiplicative range of $\mm$'s singular values, and hide polylogarithmic factors. For \cite{recht2011simpler, CandesP10}, current SDP solvers with $m$ constraints use $\Omega(m^\omega)$ time \cite{JiangKLP020, HuangJ0T022}. We use $\star$ to mean additional assumptions are made on the noise beyond a Frobenius norm bound.}
\end{figure}

\subsection{Related work}\label{ssec:related}

The literature on matrix completion is vast and a full survey is beyond our scope.
For conciseness, we only consider the most relevant work here.
Much of the algorithmic work on matrix completion falls into three categories, two of which we have already discussed in some depth.
First, there is work on solving matrix completion using SDPs such as nuclear norm minimization, e.g.,~\cite{candes2010power,CandesP10,CandesR12,recht2011simpler,ding2020leave}.
These algorithms typically attain strong statistical guarantees, but have superlinear runtimes in the problem dimensionality.
Second, there is the line of work on formally analyzing nonconvex methods such as alternating minimization, e.g.,~\cite{keshavan2010matrix,hardt2014understanding,HardtW14,JainN15,zhao2015nonconvex,sun2016guaranteed,cherapanamjeri2017nearly,zhang2019correction}.
While these achieve runtimes which are linear in the dimension of the problem, all prior results incurred large polynomial factors of $r$ or other problem parameters in their runtime (and sometimes their sample complexity as well).
We also remark that many of these papers consider notions of robust matrix completion, but tend to consider the setting where the noise matrix is sparse as opposed to norm-bounded, which is the setting we consider.

Finally, there is also the line of work on analyzing convex methods such as gradient descent for matrix completion.
In many of those works, the objective is to demonstrate the more qualitative result that the optimization landscape for matrix completion has no spurious local minima~\cite{sun2015nonconvex,de2015global,ge2016matrix,jin2016provable,zhang2018fast,zhang2022accelerating}.
Consequently, their quantitative guarantees tend to be somewhat loose compared to results using convex programming or nonconvex methods.
Additionally, because these methods are based on gradient descent, they tend to have runtimes which scale polynomially with the condition number of the underlying matrix.
In contrast, our algorithms run in time which is polylogarithmic in the condition number.
One notable exception is~\cite{zhang2022accelerating}; however, this paper only proves local convergence results for their method.

\subsection{Overview of approach}\label{ssec:approach}

In this section, we overview the two main components of our matrix completion algorithms: our iterative method for partial matrix completion (given in Section~\ref{sec:partial}) and our recovery algorithm for the missing row and column subsets which our iterative method fails to give guarantees on (given in Section~\ref{sec:fixing}). Throughout this discussion we let $\mms \defeq \R^{n \times n}$ be a rank-$\rs$ matrix which we wish to recover to disambiguate from iterates denoted as $\mm$; we also let $m = n$ for simplicity. We delay discussion of the noise-robustness of our matrix completion algorithms to the end of the section.

\subsubsection{Partial matrix completion} 

\paragraph{Short-flat decompositions.} Our partial matrix completion algorithm is motivated by a recent approach to sparse recovery developed in \cite{KelnerLLST22}. This approach iteratively makes progress towards recovering a sparse target vector $\xs$ by taking projected gradient steps. The key observation of \cite{KelnerLLST22} is that in the sparse recovery setting, the gradient of the least-squares objective is decomposable into an $\ell_2$-bounded component (the signal direction towards $\xs$) and an $\ell_\infty$-bounded component (the noise), termed a ``short-flat decomposition.'' The algorithm of \cite{KelnerLLST22} carefully used truncation onto the set of sparse vectors (which enjoys a bounded $\ell_1$-to-$\ell_2$ ratio), along with the $\ell_1$-$\ell_\infty$ H\"older's inequality, to bound how much the flat noise component inhibits progress.

We now give a first attempt at executing this strategy for matrix completion, noting that the set of low-rank matrices is a spectral analog of the set of sparse vectors. Let $\mm \in \R^{n \times n}$ be a current iterate, assume it is rank-$\rs$ (for simplicity), and let $\Omega \subseteq [n] \times [n]$ be a uniformly random set of indices with $|\Omega| \approx pn^2$, where $p$ is the observation probability. Suppose we are promised $\normf{\mm - \mms} \le 1$. A natural descent step balancing the goals of making progress towards $\mms$ and maintaining that our iterate has low rank takes $\md \gets [\mms - \mm]_\Omega$ to be the observed difference matrix, lets $\mg$ be the rank-$O(\rs)$ truncation of the SVD of $\md$, and updates $\mm' \gets \mm + \frac \eta p \mg$ for an appropriate step size $\eta > 0$. If $\md$ sufficiently approximates $\mms - \mm$ in the operator norm (up to $\approx (\rs)^{-\half}$), it is straightforward to adapt arguments of \cite{KelnerLLST22} to show that this step makes substantial progress in decreasing distance to $\mms$, e.g., $\normsf{\mm' - \mms} \le \half$.
The intuition for this argument is that
\begin{equation}\label{eq:short_flat_intro}
\frac 1 p \md = \underbrace{\mms - \mm}_{\defeq \mx} + \underbrace{\Par{\frac 1 p [\mms - \mm]_{\Omega} - (\mms - \mm)}}_{\defeq \my}.
\end{equation}
In this decomposition, note that $\mx$ is low-rank and exactly in the signal direction $\mms - \mm$, so if we could remove the influence of $\my$ then the rank-$2\rs$ truncation of $\mx$ (indeed, even no truncation at all) would exactly take us towards $\mms$. Moreover, if we could bound the operator norm of the noise component $\my$, then applying perturbation arguments such as Weyl's theorem shows that $\my$ cannot affect the progress direction by too much after truncating $\md$'s SVD. Furthermore, assuming the random samples $\Omega$ are independently drawn,\footnote{We show how to lift this assumption by splitting samples and using them iteratively as holdouts in Lemma~\ref{lem:multiple_obs}.} it is straightforward to see that $\my$ is mean-zero, so we can hope to control its operator norm using concentration bounds such as the matrix Bernstein inequality. This argument parallels the strategy of \cite{KelnerLLST22}, where we may think of $\mx$ as the short progress component and $\my$ as the flat noise component (each in a singular value sense).

\paragraph{Bounding the difference matrix.} Unfortunately, without further assumptions, the operator norm of $\my$ may be too large. A hard example is when $\mms = uu^\top$ and $\mm = vv^\top$ where $u$, $v$ have entries in $\pm n^{-\half}$ differing in only one coordinate.  In this example, a randomly sampled $\my$ (after debiasing via rescaling by the inverse sampling probability $\approx n$, as in \eqref{eq:short_flat_intro}) will have constant rank and operator norm. A natural way to prove an operator norm bound on such a randomly sampled matrix is via the matrix Bernstein inequality, which shows that we obtain the desired bounds if the difference $\mm - \mms$ has row and column norms bounded by $\approx n^{-\half}$ and entries bounded by $\approx \sqrt{\rs} \cdot n^{-1}$ (see Lemma~\ref{lem:random_sampling_error}); these conditions fail in our hard example as it has one row and column norm which is too large. Nevertheless, Markov's inequality shows that in general, only a constant fraction of rows and columns of the difference matrix can have norms which are too large; these subsets can then be 
estimated from observations and 
dropped (carried out in Section~\ref{ssec:removal}).

This leaves the issue of large entries, a second obstacle for our matrix Bernstein argument. With no assumptions on the row and column spans of $\mms$, it is possible that $\mm - \mms$ has a few large entries missed by our random observations which can ruin our bound on $\my$. We first show that due to the rank bound on $\mm - \mms$, these large entries must be localized to small (unknown) subsets of rows and columns (Lemma~\ref{lem:cover-large-entries}). We then introduce a new measure of progress (Definition~\ref{def:partial-closeness}) where we say two matrices are close if their difference has small Frobenius norm on a large submatrix, which allows us to exclude these small unknown subsets with large entries. Finally, we are able to prove our iterative method makes progress in this modified notion of distance, and thus achieves partial completion. We give a complete statement of the guarantees of our partial matrix completion method in Proposition~\ref{prop:iterative-step}, and demonstrate how to use it recursively to obtain Theorem~\ref{thm:pmc-informal} in Section~\ref{ssec:recurse}.


\paragraph{Mitigating rank blowup.} One technical issue which arises in our partial completion method is that, roughly speaking, the rank of our iterate $\mm$ increases by a constant factor in each iteration. Our earlier argument relied on a rank bound on $\mm$, so this rank blowup is problematic. If our progress measure were $\normf{\mm - \mms}$ (i.e., an exact distance bound), we could simply truncate the SVD of $\mm$ to project it onto the set of low-rank matrices, which affects our progress by a constant factor. However, our guarantee is with respect to a modified notion of distance, so this does not hold. Instead, we show that we can make substantially more progress by taking slightly more samples, cutting the modified distance measure by a factor of $\approx \exp(\sqrt{\log(\rs)}) = (\rs)^{o(1)}$ in each iteration, so that in $\approx \sqrt{\log(\rs)}$ iterations we have made a polynomial factor progress. This results in only a $(\rs)^{o(1)}$ factor blowup in the rank of our iterate, and we then apply our fixing procedure (discussed next) to reduce the rank. For our self-contained partial completion result (Theorem~\ref{thm:pmc-informal}), which is performed in one shot without a fixing step, the corresponding overhead is a factor of $n^{o(1)}$.


\subsubsection{From partial completion to full completion}

\paragraph{Finding a representative subset.} Our distance measure in our partial completion algorithm (see e.g., Theorem~\ref{thm:pmc-informal}) allows for the subsets on which we make progress to be unknown, but this causes issues when used for full completion. Indeed, our partial completion method made no assumptions about the regularity of $\mms$, but to recover dropped subsets (as well as subsets excluded by our distance measure) we need to impose structural assumptions. For simplicity in the following discussion, assume $\mms$ has $(\Omega(1), \Omega(1))$-regular row and column spans for appropriate constants (Definition~\ref{def:regular-subspace}). We also assume for simplicity that $\mm$, the output of our partial completion method, satisfies $[\mm]_{A, B} = [\mms]_{A, B}$ for $|A|, |B| \ge 0.99n$ exactly, i.e., we have run the partial completion method to high accuracy. Finally, we ignore the effect of explicitly dropped rows and columns, as these can be recovered analogously to the (unknown) excluded subsets in our distance measure.

Our high-level strategy is to identify a set $T$ of $\approx \rs$ columns of $\mm$, such that $\mm = \mms$ exactly on these columns, and the column space of $\mms_{:T}$ spans the column space of $\mms_{:T}$. We call such a set $T$ ``representative'' with respect to $(\mm, \mms)$, defined formally in Definition~\ref{def:representative-subset} (which includes additional parameters when $\mm_{A \times B}$ is only close to $\mms_{A \times B}$, rather than exactly equal). We begin with a preprocessing phase in Section~\ref{ssec:sparsify_errors}, where we drop any rows and columns upon which we observe empirical errors. This guarantees that on the remaining submatrix, the difference matrix $\mm - \mms$ has at most $\frac {0.01\rs} n$ nonzero entries per row or column (else they would have been dropped).

We next provide a structural fact that any rank-$\rs$ matrix with such bounded row and column sparsity must have all of its errors localized to a $1\% \times 1\%$ submatrix (see Lemma~\ref{lem:most-rows-good} for a formal statement which handles noise). In the noiseless case, this fact follows straightforwardly from a Gram-Schmidt argument (Lemma~\ref{lem:low-rank-sparse}). This implies that a majority of the remaining columns of $\mm$ and $\mms$ (after preprocessing) are actually identical, and are thus valid to include in a representative subset. We further develop a tester for verifying whether a given column $j \in [n]$ should be included in our representative subset, by drawing $\approx \rs$ random columns of our iterate $\mm$ and checking whether column $\mm_{:j}$ is contained in the span of these random columns. This test is motivated by the observation that if $\mm_{:j}$ contains a sparse error (and hence should not be included), with constant probability our random sample will dodge this error due to our preprocessing step, and hence $\mm_{:j}$ will not be contained in its span. By repeating our tester a small number of times, we can ensure the subset of columns we include is representative.

\paragraph{Regression with a representative subset.} Once we have determined a representative subset $T$, it suffices to use our regularity assumptions to argue that $\approx \rs$ random observations of any column of $\mms$ uniquely determine how it can be completed as a linear combination of $\mm_{:T} = \mms_{:T}$. In the noiseless case, this means that we can simply solve roughly $n$ regression problems in $\rs \times \rs$ matrices to fully complete the matrix. Our formal definition of a representative subset contains a quantitative bound ensuring $\mms_{:T}$ spans the column space of $\mms$ in a well-conditioned manner. This allows for us to argue about the generalization error of our regression subroutines under noise.

We remark that if after our partial completion subroutine, we knew which row and column subsets $A, B$ our iterate was close to $\mms$ on, we could directly skip to this regression step for recovering poorly-behaved subsets. Handling the potential of sparse errors on unknown subsets of our iterate in a noise-tolerant way constitutes the bulk of our technical development in Section~\ref{sec:fixing}.

\subsubsection{Robust matrix completion}

Finally, we discuss how our framework extends to the noisy setting in a natural way. In general, our fixing step in Section~\ref{sec:fixing} takes as input $\mm$ with the guarantee that $\mm$ is $\tDelta$-close to $\mms$ on a submatrix (see Definition~\ref{def:partial-closeness}), after excluding an $\frac \alpha 2$-fraction of rows and columns explicitly dropped by our iterative method, and an additional $\frac \alpha 2$-fraction due to our distance measure (where $\alpha$ is a subspace regularity parameter). Assuming $\tDelta$ is sufficiently larger than $\normf{\mn}$, where we receive observations from $\mms + \mn$ (i.e.\ $\mn$ is the noise), our fixing step learns any excluded rows and columns to a comparable distance to the average undropped row or column, and yields a standard distance guarantee (rather than a partial one). However, this stronger standard distance guarantee comes at the cost of a $\textup{poly}(\rs)$ overhead over the initial distance promise $\tDelta$, and is stated formally in Proposition~\ref{prop:fix}. This overhead is due to lossiness when converting between operator norm distance guarantees (which naturally arises in analyzing the generalization error of our regression step), and Frobenius norm distance guarantees (which our iterative method yields). By ensuring that all matrices encountered throughout are low-rank, this lossiness only has $\rs$-dependent factors. 

Our robust matrix completion results stated in Theorem~\ref{thm:informal_regular} and Corollary~\ref{cor:informal_incoherent} follow by applying the guarantees of Proposition~\ref{prop:iterative-step} (our partial matrix completion algorithm) and Proposition~\ref{prop:fix} (our fixing step) recursively. By running Proposition~\ref{prop:iterative-step} for a small number of steps to control the blowup of our iterate's rank, and applying Proposition~\ref{prop:fix} to reduce the rank and recover dropped subsets, we can make multiplicative distance progress towards our noise threshold $\Delta \ge \normf{\mn}$. The error overhead incurred by our algorithms is then due to a final application of Proposition~\ref{prop:fix}.
\section{Preliminaries}\label{sec:prelims}

\paragraph{General notation.} Throughout $[n] \defeq \{i \in \N \mid i \le n\}$. When $S \subseteq T$ and $T$ is clear from context, we let $S^c \defeq T \setminus S$.  We say $v \in \R^d$ is $s$-sparse if it has at most $s$ nonzero entries. Applied to a vector, $\norm{\cdot}_p$ is the $\ell_p$ norm. The Frobenius, operator, and trace norms of a matrix are denoted $\normf{\cdot}$, $\normop{\cdot}$, and $\normtr{\cdot}$ and correspond to the $2$-norm, $\infty$-norm, and $1$-norm of the singular values of a matrix. The all-zeroes and all-ones vectors of dimension $d$ are denoted $\0_d$ and $\1_d$.

\paragraph{Matrices.} Matrices are denoted in boldface. We equip $\R^{m \times n}$ with the inner product $\inprod{\ma}{\mb} \defeq \Tr(\ma^\top \mb)$. The $d \times d$ identity matrix is denoted $\id_d$, and the all-zero $m \times n$ matrix is denoted $\mzero_{m \times n}$. The ordered singular values of $\mm \in \R^{m \times n}$ with $m \ge n$ are denoted $\{\sigma_i(\mm)\}_{i \in [n]}$, where $\sigma_1$ is largest and $\sigma_n$ is smallest; when $\mm \in \R^{d \times d}$ is symmetric, we similarly define $\{\lam_i(\mm)\}_{i \in [d]}$. When $i$ is larger than the rank of $\mm$, $\sigma_i(\mm) \defeq 0$. The number of nonzero entries of $\mm$ is denoted $\nnz(\mm)$, and the largest absolute value among its entries is denoted $\norm{\mm}_{\max}$. For $\tau \ge 0$, $\mm \in \R^{m \times n}$, we let $\mm^{\le \tau}$ be such that $\mm^{\le \tau}_{ij}$ is the median of $-\tau$,$\tau$, and $\mm_{ij}$. We say $\mm \in \R^{m \times n}$ is given as a rank-$r$ factorization if we have explicit access to $\mmu \in \R^{m \times r}$, $\mv \in \R^{n \times r}$ with $\mm = \mmu \mv^\top$. For symmetric positive semidefinite $\ma, \mb \in \R^{d \times d}$ we use $\ma \approx_\eps \mb$ to denote $\exp(-\eps)\mb \preceq \ma \preceq \exp(\eps)\mb$. When $\ma$ is symmetric positive definite we let $\kappa(\ma)$ be the ratio of its largest and smallest eigenvalues. We define $\tmv(\mm)$ as the amount of time it takes to compute $\mm v$ for any $v$; note $\tmv(\mm) = O(\nnz(\mm))$, and if $\mm \in \R^{m \times n}$ is given as a rank-$r$ factorization then $\tmv(\mm) = O((m + n)r)$.

\paragraph{Submatrices.} For $\mm \in \R^{m \times n}$ and subsets $S \subseteq [m]$, $T \subseteq [n]$, the matrix $\mm_{S, T}$ denotes the $|S| \times |T|$ submatrix of $\mm$ restricted to rows $S$ and columns $T$. When $A = \{i\}$ for $i \in [m]$, we abbreviate this as $\mm_{i, B}$, and similarly define $\mm_{A, j}$ for $j \in [n]$. For $\mm \in \R^{m \times n}$ we write $\mm_{A:}$ as shorthand for $\mm_{A, [n]}$ and $\mm_{:B}$ for $\mm_{[m], B}$. The $i^{\text{th}}$ row and $j^{\text{th}}$ column of $\mm$ are similarly denoted $\mm_{i:}$ and $\mm_{:j}$. When dimensions are clear, the matrix which is all-zeroes except for a one in the $(i, j)^{\text{th}}$ entry is $\me_{ij}$ and $e_i$ is the $i^{\text{th}}$ standard basis vector. We say that $\mn$ is a $\gamma$-submatrix of $\mm \in \R^{m \times n}$ if $\mn = \mm_{S, T}$ for $S \subseteq [m]$, $T \subseteq [n]$ with $|S| \ge m - \gamma\min(m, n)$ and $|T| \ge n - \gamma\min(m, n)$. We say that $\mm$ is $s$-row column sparse (RCS) if each row and column of $\mm$ has as most $s$ nonzero entries. When $\Omega \subseteq [m] \times [n]$ is a set of index pairs, $\mm_\Omega$ zeroes out all entries in $\mm$ indexed by $\Omega^c$ (we similarly define $v_\Omega$ for vectors $v \in \R^d$ and $\Omega \subseteq [d]$).

\paragraph{Comparing matrices.} We introduce two nonstandard notions of closeness between matrices. These notions will be used primarily in stating the guarantees of our subroutines in Sections~\ref{sec:partial} and~\ref{sec:fixing} respectively, to deal with subsets or sparse error patterns out of our control.

\begin{definition}[Closeness on a submatrix]\label{def:partial-closeness}
We say $\mm, \mm' \in \R^{m \times n}$ are $\Delta$-close on a $\gamma$-submatrix if there exist subsets $A \subseteq [m]$, $B \subseteq [n]$ satisfying $|A| \ge m - \gamma \min(m,n)$, $|B| \ge n - \gamma \min(m,n)$, and 
\[\normf{\Brack{\mm - \mm'}_{A, B}} \le \Delta.\]
\end{definition}

\begin{definition}[Closeness away from an RCS matrix]\label{def:sparse-close}
We say $\mm, \mm' \in \R^{m \times n}$ are $\Delta$-close away from an $s$-RCS matrix if $\mm - \mm' = \mx + \my$, for some $\normf{\mx} \leq \Delta$, and $s$-RCS $\my$.
\end{definition}

We note that in Definition~\ref{def:partial-closeness}, the sets $A$, $B$ are unknown; similarly, in Definition~\ref{def:sparse-close}, the factorization $\mx$, $\my$ is unknown. Our analysis will only use these definitions as existential statements.

\paragraph{Observation model.} For $\mm \in \R^{m \times n}$, we specialize the notation $\mm_\Omega \gets \oracle_p(\mm)$ to mean $\Omega \subset [m] \times [n]$ contains each $(i, j) \in [m] \times [n]$ with probability $p$ (sampled independently), and $\mm_\Omega$ is the sum of the observations $\mm_{ij}\me_{ij}$ for $(i, j) \in \Omega$. When an algorithm requires the ability to query $\mm \in \R^{m \times n}$ with $\oracle_p$ for various $p$ (specified in the algorithm description), we list the input as $\oracle_{[0, 1]}(\mm)$, which also gives access to $\orzo(\mm_{S, T})$ for $S \subseteq [m]$, $T \subseteq [n]$. 

We note that this observation model (querying $\oracle_p$, possibly multiple times independently) is compatible with the standard model in the literature (which only allows for a one-shot set of realized observations), up to a small loss in parameters. This is made formal through the following lemma, which shows how to simulate $K$ draws from $\oracle_p$ given one-time access to $\oracle_{Kp}$.

\begin{lemma}\label{lem:multiple_obs}
Let $\{p_k\}_{k \in [K]} \in (0, 1)$ satisfy $p_k \le p \le \frac 1 K$ for all $k \in [K]$, and let $\mm \in \R^{m \times n}$. We can simulate sequential access to $\oracle_{p_k}(\mm)$ for all $k \in [K]$ with access to $\oracle_{Kp}(\mm)$.
\end{lemma}
\begin{proof}
The probability that the entry is revealed in any of the independent, sequential queries is
\[\ptot \defeq 1 - \prod_{k \in [K]}(1 - p_k) \le Kp.\]
The conclusion then follows from two observations. First, letting $q \ge p$ satisfy $1 - (1 - q)^K = Kp$, if $\oracle_{Kp}$ reveals an entry we can efficiently simulate how many of $K$ calls to $\oracle_q$ would have revealed that entry conditioned on at least one call resulting in a reveal. Second, given access to $\oracle_q$ we can simulate $\oracle_{p_k}$ for any $p_k \le q$ by rejecting a revealed entry with the appropriate probability.
\end{proof}

In other words, Lemma~\ref{lem:multiple_obs} allows us to draw observations from a matrix a single time, and then split the samples in a way that simulates multiple sequential accesses to the matrix. 

\paragraph{Subspaces.} For a subspace $V \subseteq \R^d$ of dimension $r$, we denote its orthogonal complement by $V_\perp$. We let $\proj_V \in \R^{d \times d}$ be the projection matrix onto $V$. We let $\mb_V \in \R^{d \times r}$ denote an arbitrary matrix satisfying $\mb_V\mb_V^\top = \proj_V$ and $\mb_V^\top\mb_V = \id_r$. We say $\mmu \msig \mv^\top$ is the singular value decomposition (SVD) of $\mm$ if $\mmu, \mv$ have orthonormal columns and $\msig$ is nonnegative and diagonal; when this is not unique, we take an arbitrary SVD. We recall our definition of a regular subspace in Definition~\ref{def:regular-subspace}. We will mainly use this definition through the following equivalence.

\begin{lemma}\label{lem:regular-wc}
Let $V \subseteq \R^d$ have dimension $r$, and let $\{b_i\}_{i \in [d]} \subset \R^r$ be rows of an (arbitrary) choice of $\mb_V$. $V$ is $(\alpha,\beta)$-regular if and only if for every $S \subseteq [d]$ with $|S| \ge (1 - \alpha) d$,
\[ \beta^2\id_r \preceq \sum_{i \in S} b_ib_i^\top \preceq \id_r.\]
\end{lemma}
\begin{proof}
First observe that $\norm{\proj_V v}_2^2 + \norm{\proj_{V_\perp} v}_2^2 =\norm{v}_2^2$ and $\norm{\proj_V v}_2^2 = v^\top \proj_V v= \norm{\mb_V v}_2^2$ for all $v \in \R^d$. Consequently, $V$ is $(\alpha,\beta)$ regular if and only if $\|\sum_{i \in [d]} b_i v_i\|_2^2 \leq (1 - \beta^2) \norm{v}_2^2$, for all $\alpha d$-sparse $v \in \R^d$. This is equivalent to the condition that for all $T \subseteq [d]$ with $|T| \leq \alpha d$ and (not necessarily sparse) $v \in \R^d$, $\|\sum_{i \in T} b_i v_i\|_2^2 \leq (1 - \beta^2) \sum_{i \in T} v_i^2$. Equivalently, for every $T \subseteq [d]$ with $|T| \leq \alpha d$, the matrix $\mb_{T:}$ must have operator norm $\le \sqrt{1 - \beta^2}$, so $\sum_{i \in T} b_i b_i^\top \preceq (1 - \beta^2)  \mI_r$. Since $\sum_{i \in S} b_i b_i^\top = \mI_r - \sum_{i \in S^c} b_i b_i^\top$ and $\sum_{i \in [d]} b_i b_i^\top = \mI_r$, the result follows.
\end{proof}

We also introduce a notion of a standard subspace in Definition~\ref{def:standard-subspace}, which is more compatible with the aformentioned incoherence assumption in the matrix completion literature. This definition is used to streamline the application of the tools from Section~\ref{sec:fixing}.

\begin{definition}[Standard subspace]\label{def:standard-subspace}
We say a subspace $V \subseteq \R^d$ of dimension $r$ is \emph{$(\alpha, \beta,\mu)$-standard} if it is $(\alpha, \beta)$-regular and there exists a subset $S \subseteq [d]$ with $|S| \geq (1 - \frac \alpha 3)d$ such that for all $i \in S$, $\norm{\proj_V e_i}_2 \leq \sqrt{\frac {\mu r} d}$.
\end{definition}

The following fact is immediate by Markov's inequality, $\norm{\proj_V e_i}_2 = \norm{\mb_V e_i}_2$, and $\normf{\mb_V}^2 = r$.
\begin{fact}\label{fact:mu_alpha}
If a subspace $V \subseteq \R^d$ is $(\alpha , \beta)$-regular, then it is $(\alpha,\beta, \frac 3 \alpha)$-standard.
\end{fact}
Thus, whenever we mention a subspace being $(\alpha , \beta,\mu)$-standard, we may assume $\mu \leq \frac 3 \alpha$. Finally, for comparison to the matrix completion literature, we also give the definition of incoherence which is typically used to parameterize algorithms.

\begin{definition}[Incohererent subspace]\label{def:incoherence}
We say a subspace $V \subseteq \R^d$ of dimension $r$ is \emph{$\mu$-incoherent} if $\norm{\proj_V e_i}_2 \le \sqrt{\frac{\mu r}{d}}$ for all $i \in [d]$.
\end{definition}

The following is then immediate from the characterization in Lemma~\ref{lem:regular-wc}. 

\begin{fact}\label{fact:incoherent-standard}
If a subspace $V \subseteq \R^d$ is $\mu$-incoherent, it is $(\frac{3}{4\mu r},\half,\mu)$-standard.
\end{fact}
\begin{proof}
Note that $\normsop{\sum_{i \in S^c} b_i b_i^\top} \le |S^c| \max_{i \in S^c} \norm{b_i}_2^2$ and apply Weyl's perturbation theorem. 
\end{proof}

We introduce the notion of a standard subspace primarily for technical convenience as it captures the parameters of both subspace regularity and incoherence.  We will prove a result (Theorem~\ref{thm:main}) in terms of all of these parameters $\alpha, \beta, \mu$ and then deduce our results for subspace regularity and incoherence by combining Theorem~\ref{thm:main} with Fact~\ref{fact:mu_alpha} and Fact~\ref{fact:incoherent-standard} respectively.

\paragraph{Concentration.} We use the following concentration inequalities and their scalar specializations.

\begin{fact}[Matrix Chernoff, Theorem 5.1.1 \cite{Tropp15}]\label{fact:matchern}
Let $\{\mx_i\}_{i \in [n]}$ be independent, $d \times d$ positive semidefinite, matrix-valued random variables satisfying $\normop{\mx_i} \le R$ with probability $1$ for all $i \in [n]$, and let $\mx$ denote their sum. For any $\eps \in (0, 1)$,
\begin{align*}\Pr\Brack{\lam_{\min}\Par{\mx} \le (1 - \eps)\lam_{\min}\Par{\E\mx} } &\le d\exp\Par{-\frac{\eps^2 \lam_{\min}(\E \mx)}{3R}}, \\
\Pr\Brack{\lam_{\max}\Par{\mx} \ge (1 + \eps)\lam_{\max}\Par{\E\mx} } &\le d\exp\Par{-\frac{\eps^2 \lam_{\max}(\E \mx)}{3R}}.
\end{align*}
\end{fact}

\begin{fact}[Matrix Bernstein, Theorem 1.6.2 \cite{Tropp15}]\label{fact:matbern}
Let $\{\mx_i\}_{i \in [n]}$ be independent, $d_1 \times d_2$ matrix-valued random variables satisfying $\E \mx_i = \mzero_{d_1 \times d_2}$ and $\normop{\mx_i} \le R$ with probability $1$ for all $i \in [n]$, let $\mx$ denote their sum, and let 
\[\sigma^2 \defeq \max\Par{\normop{\sum_{i \in [n]} \E \mx_i\mx_i^\top},\; \normop{\sum_{i \in [n]} \E \mx_i^\top\mx_i}}.\]
Then for all $t \ge 0$, $\Pr[\normop{\mx} \ge t] \le (d_1 + d_2)\exp\Par{-\frac{t^2}{2\sigma^2 + \frac 2 3 Rt}}$, so for all $\delta \in (0, 1)$,
\[\Pr\Brack{\normop{\mx} \ge \max\Par{2\sigma\sqrt{\log\Par{\frac{d_1 + d_2}{\delta}}},\; \frac{4R}{3}\log\Par{\frac{d_1 + d_2}{\delta}}}}\le \delta.\]
\end{fact}
\section{Partial matrix completion}\label{sec:partial}

In this section, we give a novel subroutine for making partial progress towards a target low-rank matrix $\mms \in \R^{m \times n}$ (whose rank is denoted $\rs$), from which we can query noisy observations. In particular, the method we develop in this section only assumes the target matrix is low-rank, without any requirement of subspace regularity in the vein of Definition~\ref{def:regular-subspace}. However, our guarantees are with respect to a weaker notion of progress, which involves explicitly dropping or excluding a small number of poorly-behaved rows and columns. 

The main result of this section is the following Proposition~\ref{prop:iterative-step}, which gives a guarantee on Algorithm~\ref{alg:descent-step} (which builds upon Algorithm~\ref{alg:filter}, a preprocessing subroutine which we explain shortly). Our Algorithm~\ref{alg:descent-step} takes as parameters $\gdrop$ and $\gadd$, as well as a matrix $\mm$ which is $\Delta$-close to $\mms$ on a $\gamma$-submatrix. It then explicitly drops roughly a $\gamma$ fraction of rows and columns which it makes no guarantees on, adds $\gadd$ to the submatrix parameter, and triples the rank of $\mm$. In return, it cuts the distance on a $(\gamma + \gadd)$-submatrix by a factor of $\ell$.

\begin{restatable}{proposition}{restatedescent}\label{prop:iterative-step}
Let $\Delta \ge 0$, $\gamma, \gadd, \delta \in (0, 1)$, and $\ell \ge 1$. Let $\mmh \defeq \mms + \mn \in \R^{m \times n}$ for $m \ge n$, $\mms$ which is rank-$\rs$, and $\mn$ satisfying $\normf{\mn} \le \frac{\Delta}{20\ell}$. If rank-$r$ $\mm \in \R^{m \times n}$ is $\Delta$-close to $\mms$ on a $\gamma$-submatrix and given as a rank-$r$ factorization,
Algorithm~\ref{alg:descent-step} returns $\tmm \in \R^{m \times n}$ as a rank-$3(r + \rs)$ factorization and $S \subseteq [m]$, $T \subseteq [n]$ satisfying the following with probability $\ge 1 - \delta$.
\begin{enumerate}
    \item $|S| \ge m - \gdrop n$, $|T| \ge (1 - \gdrop)n$, for $\gdrop = \max(400\gamma\log(m),\; 10^5 \ell^2(\gamma + \gadd))$.
    \item $\tmm_{S, T}$ is $\frac \Delta \ell$-close to $\mms_{S, T}$ on a $(\gamma + \gadd)$-submatrix.
\end{enumerate}
Algorithm~\ref{alg:descent-step} uses $O(mnp(r + \rs))$ time and one call to $\oracle_p(\mmh)$ where for a sufficiently large constant,
\[p = O\Par{\frac{(r + \rs)\ell^2}{n} \cdot \frac{\gamma + \gadd}{\gadd^2}\log^2\Par{\frac{m}{\delta}}}.\]
\end{restatable}

In Section~\ref{ssec:removal}, we begin by analyzing Algorithm~\ref{alg:filter} ($\Filter$), a preprocessing step for setting aside roughly a $\gadd$ fraction of poorly-behaved rows and columns from empirical observations. In Section~\ref{ssec:progress}, we then use the control that this preprocessing step affords over the remaining rows and columns to analyze our main iterative step, Algorithm~\ref{alg:descent-step} ($\Descent$), and prove Proposition~\ref{prop:iterative-step}. Finally, to illustrate a typical use case of Proposition~\ref{prop:iterative-step} for partial matrix completion (which reflects its use in our final algorithm), we give a self-contained result in Section~\ref{ssec:recurse} only relying on recursive use of Algorithm~\ref{alg:filter}, without the use of subspace regularity assumptions.

\subsection{Row and column removal}\label{ssec:removal}

The first step is to remove some rows and columns whose norm in $\mmh - \mm$ is too large.  This is useful because we would like to use $\mmh - \mm$ to guide the direction of our steps but we only have partial observations of it.  The rows and columns with large norms can ruin the spectral concentration of the empirical observations, so removing them allows us to prove spectral closeness between the empirical and true difference matrices. Before analyzing our removal algorithm, we state a simple concentration inequality we will use in its proof about the error of empirical norm estimates.

\begin{lemma}\label{lem:empirical_estimate}
Let $p, \delta \in (0, 1)$, let $v \in \R^d$ have $\norm{v}_\infty \le \tau$ and let $\tv \in \R^d$ have each entry $\tv_i$ independently set to $v_i$ with probability $p$, and $0$ otherwise. Then with probability $\ge 1- \delta$,
\[\Abs{\norm{v}_2^2 - \frac 1 p \norm{\tv}_2^2} \le \max\Par{\frac 1 {10}\norm{v}_2^2,\;\frac{30\tau^2\log \frac 2 \delta}{p}}.\]
\end{lemma}
\begin{proof}
By Fact~\ref{fact:matchern} with (scalar) $x_i \gets \frac 1 p \tv_i^2$, so $\E \sum_{i \in [d]} x_i = \norm{v}_2^2$, with probability $\ge 1 - \delta$,
\[\Abs{\norm{v}_2^2 - \frac 1 p \norm{\tv}_2^2} \le \frac{\tau}{\sqrt p}\norm{v}_2\sqrt{3\log \frac 2 \delta}.\]
The conclusion follows depending on which of $\frac 1 {\sqrt{10}}\norm{v}_2$ or $\frac{\tau}{\sqrt{p}}\sqrt{30\log \frac 2 \delta}$ is larger.
\end{proof}

We are now ready to state and analyze our removal process, which for logarithmically many iterations simply drops the largest rows and columns of the difference matrix, estimated from empirical observations. Our analysis proceeds in two phases. The goal of the first phase is to decrease the Frobenius norm of the true difference matrix until it is below a certain threshold, which we argue we continually make progress by concentration of the empirical observations. The second phase applies Markov's inequality to bound the number of large rows and columns once the Frobenius norm is below this threshold.

We remark that the assumed upper bound on $\tau$ in the following statement is for convenience in simplifying logarithmic terms and is not saturated in our eventual parameter settings (whereas the $\rho$ bound reflects its eventual setting). Further, the parameter $\gadd$ will eventually be set to be sufficiently small when iterating upon our algorithm, as it reflects the growth of the number of rows and columns we do not make guarantees on. To build intuiton (following discussion in Section~\ref{ssec:approach}), the reader may think of $\gamma, \gadd$ as small constants, $\tau \approx \Delta \cdot \frac{\sqrt r}{n}$, $\rho \approx \Delta \cdot \frac 1 {\sqrt n}$, and $p \approx \frac r n$.

\begin{lemma}\label{lem:filter-step}
Let $\Delta, \tau, \rho \ge 0$ and $\gamma, \gadd, p, \delta \in (0, 1)$.
Assume $\mm \in \R^{m \times n}$ is $\Delta$-close to $\mmh$ on a $\gamma$-submatrix, and that $m \ge n$. Finally, assume that
\begin{align*}
\tau \le \frac{\Delta n}{\gadd},\; \rho \ge \frac{8\Delta}{\sqrt{200\gamma n\log(\frac m {\gadd})}},\; p \ge 60\tau^2\log\Par{\frac{100m}{\delta\gadd}}\max\Par{\frac{\gamma n}{\Delta^2},\; \frac 5 {\rho^2}}.
\end{align*}
With probability $\ge 1 - \delta$, Algorithm~\ref{alg:filter} returns $S \subseteq [m]$, $T \subseteq [n]$ satisfying the following. 
\begin{itemize}
    \item $|S| \geq m - \gdrop n$, $|T| \geq (1 - \gdrop)n$, for $\gdrop = 400\gamma\log(m)$.
    \item For all $i \in S$, $\norm{[\mm - \mmh]_{i, T}^{\leq \tau }}_2 \leq \rho$, and for all $j \in T$, $\norm{[\mm - \mmh]^{\le \tau}_{S, j}}_2 \le \rho$.
    \item $\normf{[\mm - \mmh]_{S, T}^{\leq \tau }} \leq 2\Delta$. 
\end{itemize}
\end{lemma}
\begin{proof}
Throughout for convenience, we denote
\[\mds_t \defeq \Brack{\mm - \mmh}^{\le \tau}_{S_t \times T_t}
\text{ and }
\Phi_t \defeq \normf{\mds_t}^2.\]
Also, by applying Lemma~\ref{lem:empirical_estimate} with $\delta \gets \frac{\delta \gadd}{100m} \le \frac{\delta}{(m + n)(t_{\max} + 1)}$, we assume throughout the proof (giving the failure probability by a union bound) that for all iterations $0 \le t < t_{\max}$ and all $i \in S_t$, $j \in T_t$,
\begin{equation}\label{eq:empirical_close}
\begin{aligned}
\Abs{r_{i, t} - \norm{\Brack{\mds_t}_{i:}}_2^2} &\le \max\Par{\frac 1 {10}\norm{\Brack{\mds_t}_{i:}}_2^2,\; \frac{\Delta^2}{2\gamma n}},\\
\Abs{c_{j, t} - \norm{\Brack{\mds_{t}}_{:j}}_2^2} &\le \max\Par{\frac 1 {10}\norm{\Brack{\mds_{t}}_{:j}}_2^2,\; \frac{\Delta^2}{2\gamma n}},
\end{aligned}
\end{equation}
as well as (corresponding to the last round of Algorithm~\ref{alg:filter}), for all $i \in S_{t_{\max}}$ and $j \in T_{t_{\max}}$,
\begin{equation}\label{eq:empirical_close_last}
\begin{aligned}
\Abs{r_i - \norm{\Brack{\mds_{t_{\max}}}_{i:}}_2^2 } &\le \max\Par{\frac 1 {10}\norm{\Brack{\mds_{t_{\max}}}_{i:}}_2^2,\; \frac{\rho^2}{10}}, \\
\Abs{c_j - \norm{\Brack{\mds_{t_{\max}}}_{:j}}_2^2} &\le \max\Par{\frac 1 {10}\norm{\Brack{\mds_{t_{\max}}}_{:j}}_2^2,\; \frac{\rho^2}{10}}.
\end{aligned}
\end{equation}
By definition, $\Phi_0 \le mn\tau^2$, and $\Phi_t$ is nonincreasing. Next, consider an iteration $t$ where $\Phi_t \ge 4\Delta^2$. By the closeness assumption, there are $A^\star_t \subseteq S_t$, $B^\star_t \subseteq T_t$ with $|A^\star_t|, |B^\star_t| \le \gamma n$, and
\[\sum_{i \in A^\star_t} \norm{[\mds_t]_{i:}}_2^2 + \sum_{j \in B^\star_t} \norm{[\mds_t]_{:j}}_2^2 \ge \frac 3 4 \normf{\mds_t}^2 = \frac 3 4 \Phi_t.\]
Now if $\sum_{i \in A^\star_t} \norm{[\mds_t]_{i:}}_2^2 \ge \frac 3 8 \Phi_t$, by removing the $\gamma n$ largest rows by $r_{i, t}$, \eqref{eq:empirical_close} yields
\begin{align*}
\Phi_{t + 1} &\le \Par{1 - \frac{4}{5} \cdot \frac 3 8} \Phi_t + \gamma n \cdot \frac{\Delta^2}{\gamma n} \le \frac 7 {10} \Phi_t + \Delta^2 \le 0.95\Phi_t.
\end{align*}
Otherwise, $\sum_{j \in B_t^\star} \norm{[\mds_t]_{:j}}_2^2 \ge \frac 3 8 \Phi_t$, and so again $\Phi_{t + 1} \le 0.95\Phi_t$. Inducting, we thus have
\begin{align*}
\Phi_{t_{\max}} \le 4\Delta^2.
\end{align*}
Therefore, by Markov's inequality there are at most $\frac{\gdrop n}{2}$ rows in $S_{t_{\max}}$ and $\frac{\gdrop n}{2}$ columns in $T_{t_{\max} }$ with norm more than $\frac{4\Delta}{\sqrt{\gdrop n}} \le \frac \rho 2$ in $\mds_{t_{\max}}$. If a row $i \in S_{t_{\max}}$ had norm more than $\rho$ in $\mds_{t_{\max}}$, \eqref{eq:empirical_close_last} ensures it will be removed, and a similar argument holds for columns. Finally, the number of dropped rows and columns in the first $t_{\max}$ iterations is at most $\frac {\gdrop n} 2$ by our parameter choices; here we note that without loss of generality, $\gadd \ge \frac 1 m$, so $\frac m {\gadd} \le m^2$. The last condition follows since we showed $\Phi_{t_{\max}} \le 4\Delta^2$ and then dropped entries.
\end{proof}

\begin{algorithm2e}\label{alg:filter}
\caption{$\Filter(\orzo(\mmh), \mm, \tau, \rho, \Delta, \gamma, \gadd, p, \delta)$}
\DontPrintSemicolon
\codeInput $\orzo(\mmh)$, $\mm \in \R^{m \times n}$, $\tau, \rho, \Delta \ge 0$, $\gamma, \gadd, p, \delta \in (0, 1)$ \;
$S_0 \gets [m], T_0 \gets [n]$ \;
$t_{\max} \gets \lceil 20\log \frac{mn\tau^2}{4\Delta^2}\rceil$ \;
$\gdrop \gets 400\gamma\log(m)$\;
\For {$0 \le t < t_{\max}$}{
$\md_t \gets \oracle_p([\mm - \mmh]_{S_t, T_t}^{\leq \tau})$ \;
\lFor {$i \in S_t$}{
$r_{i,t} \gets \frac 1 p \norm{[\md_t]_{i:}}_2^2$ 
}
\lFor {$j \in T_t$}{
$c_{j,t} \gets \frac 1 p \|[\md_t]_{:j}\|_2^2$ 
}
$S_{t+1} \gets S_t \setminus A_t$ where $A_t \subset S_t$ corresponds to the $\gamma n$ indices $i$ with largest $r_{i, t}$  \;
$T_{t + 1} \gets T_t \setminus B_t$ where $B_t \subset T_t$ corresponds to the $\gamma n$ indices $j$ with largest $c_{i, t}$ \;
}
$\md \gets \oracle_p([\mm - \mmh]^{\le \tau}_{S_{t_{\max}}, T_{t_{\max}}})$ \;
\lFor {$i \in S_{t_{\max}}$}{
$r_i \gets \frac 1 p \norm{\md_{i:}}_2^2$ 
}
\lFor {$j \in T_{t_{\max}}$}{
$c_j \gets \frac 1 p \norm{\md_{:j}}_2^2$ 
}
$S \gets S_{t_{\max}} \setminus A$ where $A \subset S_{t_{\max}}$ corresponds to the $\frac{\gdrop n}{2}$ indices $i$ with largest $r_i$\;
$T \gets T_{t_{\max}} \setminus B$ where $B \subset T_{t_{\max}}$ corresponds to the $\frac{\gdrop n}{2}$ indices $j$ with largest $c_j$\;
\Return{$(S, T)$} \;
\end{algorithm2e}

Lemma~\ref{lem:filter-step} does not give control over entries where $\mmh - \mm$ is large.  However, below we show that the entries where $\mmh - \mm$ is large must be contained in a small number of rows and columns. We begin by observing a structural fact about entries from distinct rows and columns.

\begin{lemma}\label{lem:small_maximal_set}
Let $\mm \in \R^{m \times n}$ be rank-$r$ and let $\{(i_k, j_k)\}_{k \in [K]} \subset [m] \times [n]$ be such that $\{i_k\}_{k \in [K]}$ are distinct and $\{j_k\}_{k \in [K]}$ are distinct. Then $\sum_{k \in [K]} |\mm_{i_k, j_k}| \le \normtr{\mm}$.
\end{lemma}
\begin{proof}
Letting $\mmu \msig \mv^\top$ be an SVD of $\mm$ where columns of $\mmu$, $\mv$ are $\{u_\ell\}_{\ell \in [r]}$, $\{v_\ell\}_{\ell \in [r]}$ respectively,
\begin{align*}\sum_{k \in [K]} |\mm_{i_k, j_k}| 
&\le \sum_{k \in [K]} \sum_{\ell \in [r]} |\sigma_\ell| |u_\ell|_{i_k} |v_\ell|_{j_k} \le \sum_{\ell \in [r]} |\sigma_\ell| \Par{\half \sum_{k \in [K]} |u_\ell|_{i_k}^2 + \half \sum_{k \in [K]} |v_\ell|_{j_k}^2} \\
&\leq \sum_{\ell \in [r]} |\sigma_\ell| \Par{\half \sum_{i \in [m]} |u_\ell|_{i}^2 + \half \sum_{j \in [n]} |v_\ell|_{j}^2}
= \normtr{\mm}.\end{align*}
\end{proof}

Using Lemma~\ref{lem:small_maximal_set}, we can show that not too many distinct rows and columns of the difference between a pair of low-rank matrices which are close on a submatrix can contain very large entries.

\begin{lemma}\label{lem:cover-large-entries}
Assume rank-$r$ $\mm \in \R^{m \times n}$ and rank-$\rs$ $\mms \in \R^{m \times n}$ are $\Delta$-close on a $\gamma$-submatrix, and $m \ge n$. There are sets $A \subseteq [m]$ and $B \subseteq [n]$ such that $\norm{[\mm - \mms]_{A, B}}_{\max} \le \tau$,
\begin{align*}
|[m] \setminus A| \le \gamma n + \frac{\Delta\sqrt{r + \rs}}{\tau}
\text{ and }
|[n] \setminus B| \le \gamma n + \frac{\Delta\sqrt{r + \rs}}{\tau}.
\end{align*}
\end{lemma}
\begin{proof}
By assumption, there are $A_0 \subseteq [m]$, $B_0 \subseteq [n]$ with $|A_0| \ge m - \gamma n$, $|B_0| \ge (1 - \gamma) n$ and $\normf{[\mm - \mms]_{A_0, B_0}} \le \Delta$. Let $\{(i_k, j_k)\}_{k \in [K]} \subset A_0 \times B_0$ be maximal such that $\{i_k\}_{k \in [K]}$ and $\{j_k\}_{k \in [K]}$ contain no duplicates, and $|[\mm - \mms]_{i_k, j_k}| \ge \tau$ for all $k \in [K]$. By Lemma~\ref{lem:small_maximal_set},
\[
K\tau \leq \normtr{[\mm - \mms]_{A_0,B_0} } \leq \sqrt{r + \rs} \normf{[\mm - \mms]_{A_0,B_0} } \leq \Delta \sqrt{r + \rs}.
\]
So, $K \le \frac{\Delta\sqrt{r + \rs}}{\tau}$ and we may set $A \gets A_0 \setminus \{i_k\}_{k \in [K]}$ and $B \gets B_0 \setminus \{j_k\}_{k \in [K]}$.
\end{proof}

\subsection{Proof of Proposition~\ref{prop:iterative-step}}\label{ssec:progress}

We begin by introducing the tools we use to analyze our algorithm which proves Proposition~\ref{prop:iterative-step}. The first is a guarantee on an approximate $k$-SVD procedure from \cite{MuscoM15}.

\begin{proposition}[Theorem 1, Theorem 6, \cite{MuscoM15}]\label{prop:krylov}
Let $\mm \in \R^{m \times n}$, $k \in [\min(m, n)]$, and $\eps, \delta \in (0, 1)$. There is an algorithm $\Power(\mm, k, \eps, \delta)$ which runs in time
\[O\Par{(\nnz(\mm) k + (m + n)k^2) \cdot \frac{\log \frac {m + n} \delta}{\eps}}\]
and outputs $\mmu \in \R^{m \times r}$ with orthonormal columns such that, with probability $\ge 1 - \delta$,
\[\normop{(\id_m - \mmu \mmu^\top)\mm} \le (1+\eps)\sigma_{k + 1}(\mm).\]
\end{proposition}

The second is a bound on the operator norm error of revealing entries independently at random.

\begin{lemma}\label{lem:random_sampling_error}
Let $\mm \in \R^{m \times n}$, $p, \delta \in (0, 1)$, and suppose $\norm{\mm}_{\max} \le \tau$ and $\max_{i \in [m]} \norm{\mm_{i:}}_2 \le \rho$, $\max_{j \in [n]} \norm{\mm_{:j}}_2 \le \rho$. Let $\tmm$ be obtained by including each $(i, j) \in [d] \times [d]$ in a set $S$ with probability $p$, and setting $\tmm = \frac 1 p \sum_{(i, j) \in S} \mm_{ij} \me_{ij}$.
Then with probability $\ge 1 - \delta$,
\[\normop{\mm - \tmm} \le \max\Par{\frac{2\rho}{\sqrt p}\sqrt{\log\Par{\frac{m + n} \delta}}, \;\frac{4\tau}{3p}\log\Par{\frac{m + n}{\delta}}}.\]
\end{lemma}
\begin{proof}
For all $(i, j) \in [m] \times [n]$, define the random matrix
\[
\mx_{(i, j)} \defeq \begin{cases}
\Par{\frac 1 p - 1} \mm_{ij} \me_{ij} & \text{ with probability } p, \\
-\mm_{ij} \me_{ij} & \text{ with probability } 1 - p.
\end{cases}
\]
By definition, all $\E \mx_{(i,j)} = \mzero_{m \times n}$, and $\sum_{(i, j) \in [m] \times [n]} \mx_{(i, j)} = \mm - \tmm$, so we may apply Fact~\ref{fact:matbern}. First of all, clearly it suffices to choose $R = \frac \tau p$. Further, we bound $\sigma$: 
\begin{align*}
\sum_{(i, j) \in [m] \times [n]} \Par{p\Par{\frac 1 p - 1}^2 + 1 - p}\mm_{ij}^2 \me_{ii} &= \Par{\frac 1 p - 1} \sum_{i \in [m]} \norm{\mm_{i:}}_2^2 \me_{ii} \le \frac {\rho^2} p
\end{align*}
and a similar calculation for the other term shows $\sigma = \frac \rho {\sqrt p}$ suffices. For $t$ in the lemma statement,
\[\Pr\Brack{\normop{\mm - \tmm} \ge t} \le (m + n)\exp\Par{-\frac{t^2}{\frac{2\rho^2}{p} + \frac{2\tau t}{3p}}} \le \delta. \]
\end{proof}

The third is a bound on the Frobenius norm of a matrix which is close to an operator norm ball.

\begin{lemma}\label{lemma:short+flat-matrix}
Let $\ma, \mb \in \R^{m \times n}$ satisfy $\normop{\ma} \leq a$ and $\normf{\mb} \leq b$. If $\ma + \mb$ is rank-$r$,
\[
\normf{\ma + \mb} \leq \sqrt{2(ra^2 + b^2)}.
\]
\end{lemma}
\begin{proof}
Let the singular values of $\mm \defeq \ma + \mb$ be $\{\sigma_i\}_{i \in [r]}$. By construction, the distance from $\mm$ to the set of $m \times n$ matrices with operator norm at most $a$ is bounded by $b$, and this distance squared is $\sum_{i \in [r]} \mathbf{1}_{\sigma_i \ge a}(\sigma_i - a)^2$, so $\sum_{i \in [r]} \mathbf{1}_{\sigma_i \geq a} (\sigma_i - a)^2 \leq b^2$. The conclusion then follows from
\[
\normf{\mm}^2 = \sum_{i \in [r]} \sigma_i^2 \leq \sum_{i \in [r]} 2(a^2 + \mathbf{1}_{\sigma_i \geq a} (\sigma_i - a)^2)  \leq 2(ra^2 + b^2).
\]
\end{proof}

The last is a simple fact on singular values of a perturbed low-rank matrix.

\begin{lemma}\label{lem:lowrank+short-matrix}
If $\ma \in \R^{m \times n}$ is rank-$r$ and $\mb \in \R^{m \times n}$ satisfies $\normf{\mb} \le b$, $\sigma_{2r + 1}(\ma + \mb) \le \frac b {\sqrt r}$.
\end{lemma}
\begin{proof}
Let $V \subseteq \R^m$ span the image of $\ma$, and let $U \subseteq \R^m$ be the top-$r$ left singular vector space of $\proj_{V_\perp} \mb = \proj_{V_\perp}(\ma + \mb)$. Since $\normf{\proj_{V_\perp} \mb} \le b$, the largest singular value of $\proj_{(U \cup V)_\perp}(\ma + \mb)$ is $\le \frac b {\sqrt r}$. By the min-max principle for singular values, we have the claim (as $U \cup V$ has dimension-$2r$). 
\end{proof}

\begin{algorithm2e}\label{alg:descent-step}
\caption{$\Descent(\orzo(\mmh), \mm, \rs, \Delta, \gamma, \gadd, \delta, \ell)$}
\DontPrintSemicolon
\codeInput $\orzo(\mmh)$ for $\mmh = \mms + \mn \in \R^{m \times n}$ where $\mms$ is rank-$\rs$ and $\normf{\mn} \le \frac{\Delta}{20\ell}$, $\mm \in \R^{m \times n}$ which is $\Delta$-close to $\mms$ on a $\gamma$-submatrix, given as a rank-$r$ factorization $\mm = \mmu \mv^\top$, $\Delta \ge 0$, $\gamma, \gadd, \delta \in (0, 1)$, $\ell \ge 1$ \;
$(\tau, \rho) \gets \Par{\frac{\Delta\sqrt{r + \rs}}{\gadd n},\; \frac{\Delta}{20\ell\sqrt{(\gamma + \gadd)n}}}$\;
$(S,T) \gets \Filter(\orzo(\mmh), \mm, \tau, \rho, 1.1\Delta, \gamma, \gadd, \frac{120000(r + \rs)\ell^2}{n} \cdot \frac{\gamma + \gadd}{\gadd^2}\log(\frac{300m}{\delta\gadd}), \frac \delta 3)$ \;\label{line:filter}
$\mx \gets \oracle_q([\mmh - \mm]^{\leq \tau}_{S, T})$ for $q \gets \frac{15(r + \rs)\ell \log \frac{6m}{\delta}}{\gadd n}$ \;\label{line:reveal}
$\widehat{\mmu} \gets \Power(\mx, 2(r + \rs), 0.1, \frac \delta 3)$ (see Proposition~\ref{prop:krylov}) \;
$(\mmu', \mv') \gets (\mmu, \mv)$ with columns of $\widehat{\mmu}$, $\frac 1 q \mx^\top \widehat{\mmu}$ appended respectively \;
\Return{ $(\mmu', \mv', S,T)$} \;
\end{algorithm2e}

Now we assemble the pieces and prove Proposition~\ref{prop:iterative-step}, our main iterative method guarantee. To a large extent, the proof strategy in Proposition~\ref{prop:iterative-step} is patterned off the short-flat decomposition analysis of the iterative method in \cite{KelnerLLST22}. Specifically, we show how to decompose the difference matrix (on a large submatrix) into a Frobenius-norm bounded component and an operator-norm bounded component, which allows us to bound the effect of the error on the submatrix via Lemma~\ref{lemma:short+flat-matrix}. We restate the result here for convenience to the reader.

\restatedescent*
\begin{proof}
Throughout, we denote (in accordance with the guarantees of Lemma~\ref{lem:filter-step}):
\[\rho \defeq \frac{\Delta}{20\ell\sqrt{(\gamma + \gadd)n}},\; \tau \defeq \frac{\Delta \sqrt{r + \rs}}{\gadd n}.\]
We also denote $\mx' \defeq \widehat{\mmu}\widehat{\mmu}^\top \mx$,
and let \[\tmm \defeq \mm + \frac 1 q \mx' = \mmu'(\mv')^\top\] be the matrix whose low-rank factorization is the output of Algorithm~\ref{alg:descent-step}.
By the assumed bound on $\mn$, $\mm$ is $1.1\Delta$-close to $\mmh$ on a $\gamma$-submatrix, and hence we may apply Lemma~\ref{lem:filter-step} with the chosen parameters. Further, since $\mm$ is $\Delta$-close to $\mms$ on a $\gamma$-submatrix, Lemma~\ref{lem:cover-large-entries} applied to $[\mm - \mms]_{S, T}$ produces $A \subseteq S$, $B \subseteq T$ by deleting $\le (\gamma + \gadd) n$ rows and columns, so $[\mm - \mms]_{A, B}$ is entrywise in $[-\tau, \tau]$. We define $\mmc$ to be equal to $\mms$ on $A \times B$ and equal to $\mm$ on $(S \times T) \setminus (A \times B)$, i.e.\ (where rows and columns are permuted so $A \times B$ is on the top left)
\begin{align*}
\mmc_{S, T} = \begin{pmatrix} \mms_{A, B} & \mm_{A, T \setminus B} \\ \mm_{S \setminus A, B} & \mm_{S \setminus A, T \setminus B}\end{pmatrix}.
\end{align*}
We will prove $\normsf{[\mmc - \tmm]_{S, T}} \le \frac \Delta \ell$, and then the conclusion follows as $\mmc_{S, T} = \mms_{S, T}$ except on $(\gamma + \gadd)n$ rows and columns by Lemma~\ref{lem:cover-large-entries}, and $S, T$ drop $\le \gdrop n$ rows and columns by Lemma~\ref{lem:filter-step}. To begin, we summarize our strategy. We decompose $[\mmc - \tmm]_{S, T}$ into three parts:
\begin{equation}\label{eq:decomp_iterative_proof}
\begin{aligned}
\Brack{\mmc - \tmm}_{S, T} &= \Par{\Brack{\mmc - \mm}_{S, T}  - \Brack{\mmh - \mm}^{\le \tau}_{S, T}} \\
 &+ \Par{\Brack{\mmh - \mm}^{\le \tau}_{S, T} - \frac 1 q \mx_{S, T}} + \frac 1 q\Brack{\mx - \mx'}_{S, T}.
 \end{aligned}
\end{equation}
We will bound each of the terms in \eqref{eq:decomp_iterative_proof} (the first in Frobenius norm and the latter two in operator norm), and then apply Lemma~\ref{lemma:short+flat-matrix}. First, we claim that for all $(i, j) \in A \times B$,
\begin{align*}
\Abs{[\mmc - \mm]_{i,j} - [\mmh - \mm]_{i,j}^{\le \tau}} \le \Abs{[\mmc - \mm]_{i,j} - [\mmh - \mm]_{i,j}} = \Abs{[\mmc - \mmh]_{i,j}}.
\end{align*}
This is because $[\mmc - \mm]_{A, B}$ is entrywise in $[-\tau, \tau]$ by definition, so projecting an entry of $[\mmh - \mm]_{S, T}$ onto $[-\tau, \tau]$ only decreases the distance. Hence, we bound the first term of \eqref{eq:decomp_iterative_proof} in $A \times B$ and outside separately: since $[\mmc - \mm]_{S, T}$ vanishes outside $A \times B$,
\begin{equation}\label{eq:bound_term_1}
\begin{aligned}
\normf{\Brack{\mmc - \mm}_{S, T}  - \Brack{\mmh - \mm}^{\le \tau}_{S, T}} &\le \normf{\Brack{\mmc - \mmh}_{A, B}} \\
&+ \normf{\Brack{\mmh - \mm}^{\le \tau}_{S \setminus A, T}} + \normf{\Brack{\mmh - \mm}^{\le \tau}_{S, T \setminus B}} \\
&\le \normf{\mn} + \Par{\sqrt{|S\setminus A|} + \sqrt{|T \setminus B|}} \rho \\
&\le 2\sqrt{(\gamma + \gadd)n}\rho + \frac{\Delta}{20\ell} \le \frac{\Delta}{10\ell}.
\end{aligned}
\end{equation}
Next, by Lemma~\ref{lem:random_sampling_error}, the entrywise bound on $[\mmh - \mm]^{\le \tau}$ and the row/column bounds on $[\mmh - \mm]_{S, T}$,
\begin{equation}\label{eq:bound_term_2}
\begin{aligned}
\normop{\Brack{\mmh - \mm}^{\le \tau}_{S, T} - \frac 1 q \mx_{S, T}} &\le \max\Par{\frac{2\rho}{\sqrt{q}}\sqrt{\log\Par{\frac{3(m + n)}{\delta}}},\; \frac{4\tau}{3q}\log\Par{\frac{3(m + n)}{\delta}}} \\
&\le \frac{\Delta}{10\sqrt{r + \rs}\ell},
\end{aligned}
\end{equation}
with probability $\ge 1 - \frac \delta 3$. Finally, note that 
\begin{equation}\label{eq:low-rank+small}
\begin{aligned}
\frac 1 q \mx_{S, T} &= \Brack{\mmc - \mm}_{S, T} + \Par{\Brack{\mmh - \mm}^{\le \tau}_{S, T} - \Brack{\mmc - \mm}_{S, T}} + \Par{\frac 1 q \mx_{S, T} - \Brack{\mmh - \mm}^{\le \tau}_{S, T}}
\end{aligned}
\end{equation}
so it is the sum of a rank-$(r + \rs)$ matrix, a Frobenius norm bounded matrix (by \eqref{eq:bound_term_1}), and an operator norm bounded matrix (by \eqref{eq:bound_term_2}). Therefore,
\begin{align*}
\sigma_{2(r + \rs) + 1}\Par{\frac 1 q \mx_{S, T}} &\le \normop{\Brack{\mmh - \mm}^{\le \tau}_{S, T} - \frac 1 q \mx_{S, T}} \\
&+ \sigma_{2(r + \rs) + 1}\Par{\Brack{\mmc - \mm}_{S, T} + \Par{\Brack{\mmh - \mm}^{\le \tau}_{S, T} - \Brack{\mmc - \mm}_{S, T}}} \\
&\le \frac{\Delta}{10\sqrt{r + \rs}\ell} + \frac{\Delta}{10\sqrt{r + \rs}\ell} \le \frac{\Delta}{5\sqrt{r + \rs}\ell}.
\end{align*}
Above, the first inequality followed by Weyl's perturbation theorem, and the second followed from Lemma~\ref{lem:lowrank+short-matrix} and \eqref{eq:bound_term_1}, \eqref{eq:bound_term_2}. By Proposition~\ref{prop:krylov} we then have that
\begin{equation}\label{eq:bound_term_3}
\begin{aligned}
\normop{\frac 1 q \Brack{\mx - \mx'}_{S, T}} &\le \frac{1.1\Delta}{5\sqrt{r + \rs}\ell},
\end{aligned}
\end{equation}
 with probability $\ge 1 - \frac \delta 3$.
The decomposition \eqref{eq:decomp_iterative_proof} shows we can write $[\mmc - \tmm]_{S, T}$ as the sum of a Frobenius norm bounded matrix (the contribution of \eqref{eq:bound_term_1}) and an operator norm bounded matrix (the contributions of \eqref{eq:bound_term_2} and \eqref{eq:bound_term_3}). Further, since $[\mmc - \tmm]_{S, T} = [\mms - \mm]_{A, B} - \frac 1 q \mx'_{S, T}$ is the sum of a rank-$(r + \rs)$ matrix and a rank-$2(r + \rs)$ matrix, it is rank $3(r + \rs)$. Hence, 
\[\normf{\Brack{\mmc - \tmm}_{S, T}} \le \sqrt{6(r + \rs)} \cdot \Par{\frac{1.1\Delta}{5\sqrt{r + \rs}\ell} + \frac{\Delta}{10\sqrt{r + \rs}\ell}}+ \sqrt{2} \cdot \frac{\Delta}{10\ell} \le  \frac \Delta \ell\]
follows by applying Lemma~\ref{lemma:short+flat-matrix} with $\ma = [\mmh - \mm]_{S, T}^{\le \tau} - \frac 1 q \mx'_{S, T}$ and $\mb = [\mmc - \mm]_{S, T} - [\mmh - \mm]_{S, T}^{\le \tau}$.
The failure probability comes from a union bound over Lemma~\ref{lem:filter-step}, Lemma~\ref{lem:random_sampling_error}, and Proposition~\ref{prop:krylov}. We use Lemma~\ref{lem:multiple_obs} to upper bound the reveal probability, since Line~\ref{line:filter} requires $O(\log(\frac m {\gadd}))$ calls to $\oracle_{q'}$ for the specified $q'$, and Line~\ref{line:reveal} requires one call to $\oracle_q$ for $q = O(q')$. 

Finally, we discuss runtime. The runtime of Lines~\ref{line:filter} and~\ref{line:reveal} are bottlenecked by computing $O(mnp)$ entries of $\mmh - \mm$, where $p$ is specified in the statement of Proposition~\ref{prop:iterative-step}; a Chernoff bound implies the number of revealed entries will be within a constant factor of its expectation within the failure probability budget. Since $\mm$ is given as a rank-$r$ factorization and entries of $\mmh$ are given, this cost is $O(mnp \cdot r)$. The runtime cost of $\Power$ is specified by Proposition~\ref{prop:krylov} to be $O(mnq \cdot (r + \rs) \log \frac{m}{\delta})$, where $\tmv(\mx) = O(\nnz(\mx)) = O(mnq)$, and this does not dominate. The runtime cost of computing $\mx^\top \widehat{\mmu}$ is $O(mr^2)$ using $\mx = \mmu \mv^\top$ and also does not dominate.
\end{proof}

\subsection{Partial matrix completion via $\Descent$}\label{ssec:recurse}

In this section, we give a simple recursive application of $\Descent$ to give a self-contained result on partial matrix completion. For simplicity, we assume we have an upper bound on the largest singular value of the target matrix $\mms$; we will show how to lift this assumption in Section~\ref{sec:algos}.

\begin{algorithm2e}[ht!]\label{alg:pmc}
\caption{$\PMC(\orzo(\mmh), \rs, \sigma, \Delta, \alpha, \delta, \ell)$}
\DontPrintSemicolon
\codeInput $\orzo(\mmh)$ for $\mmh = \mms + \mn \in \R^{m \times n}$ where $\mms$ is rank-$\rs$ satisfying $\normop{\mms} \le \sigma$ and $\normf{\mn} \le \Delta$, $\alpha, \delta \in (0, 1)$, $\ell \ge 1$ \;
$\tDelta \gets \sqrt{\rs} \sigma$\;
$(\mmu, \mv) \gets (\mzero_{m \times 0}, \mzero_{n \times 0})$ \;
$k \gets 0$\;
$(S, T) \gets ([m], [n])$\;
$K \gets \lceil\log \ell\rceil$\;
$\gadd \gets \frac{\alpha}{\max(800K^2\log(m), 2 \cdot 10^5\ell^2K^2)}$\;
\While{$\tDelta \ge 20\ell\Delta \textup{ and } k \le K$}{\label{line:first_while_start}
$(\mmu, \mv, S, T) \gets \Descent(\orzo(\mmh_{S, T}), [\mmu \mv^\top]_{S, T}, \rs, \tDelta, \gadd k, \gadd, \frac \delta {2K}, \ell)$\;
$\tDelta \gets \frac \Delta \ell$\;
$k \gets k + 1$\;
}\label{line:first_while_end}
$k_{\textup{freeze}} \gets k$\;
\While{$\tDelta \ge 20e\Delta \textup{ and } k - k_{\textup{freeze}} \le K$}{\label{line:second_while_start}
$(\mmu, \mv, S, T) \gets \Descent(\orzo(\mmh_{S, T}), [\mmu \mv^\top]_{S, T}, \rs, \tDelta, \gadd k, \gadd, \frac \delta {2K}, e)$\;
$\tDelta \gets \frac \Delta e$\;
$k \gets k + 1$\;
}\label{line:second_while_end}
\Return{ $(\mmu, \mv, S,T)$} \;
\end{algorithm2e}

\begin{corollary}\label{cor:pmc}
Let $\mms \in \R^{m \times n}$ be rank-$\rs$, $\normop{\mms} \le \sigma$, $m \ge n$, $\delta \in (0, 1)$, let $\mmh = \mms + \mn$ for $\normf{\mn} \le \Delta$, and let $\ell \ge 1$. Algorithm~\ref{alg:pmc} returns $\mmu \in \R^{m \times r}$, $\mv \in \R^{n \times r}$, and $(S, T)$, for $r = \rs \textup{poly}(\ell)$, such that $[\mmu \mv^\top]_{S, T}$ is $O(\max(\sigma \sqrt{\rs} \exp(-\log^2 (\ell)), \Delta))$-close to $\mms$ on an $\alpha$-submatrix and $|S| \ge m - \alpha n$, $|T| \ge (1-\alpha)n$, with probability $\ge 1 - \delta$. Algorithm~\ref{alg:pmc} uses $O(\frac{m(\rs)^2\textup{poly}(\ell)}{\alpha} \log^3(\frac m \delta))$ 
time and one call to $\oracle_p(\mmh)$, where for a sufficiently large constant,
\[p = O\Par{\frac{\rs\textup{poly}(\ell)}{\alpha n}\log^3\Par{\frac m \delta}}.\]
\end{corollary}
\begin{proof}
The failure probability follows by applying a union bound to the $2K$ calls to $\Descent$. Next, we claim that throughout the algorithm we maintain the invariant that $\tDelta$ is an overestimate on the distance from $[\mmu\mv^\top]_{S, T}$ to $[\mms]_{S, T}$ on a $\gadd k$-submatrix; this is clearly true at the beginning of the algorithm, since $\normf{\mms} \le \sqrt{\rs}\normop{\mms}$. Further, applying Proposition~\ref{prop:iterative-step} shows that this invariant is preserved in each iteration, which gives the closeness guarantee since $\ell^{-K} \le \exp(-\log^2(\ell))$. We note that the role of the first phase (Lines~\ref{line:first_while_start} to~\ref{line:first_while_end}) is to cut the initial distance estimate by a factor $\ell^{-K}$, but is bottlenecked by the requirement that $\tDelta \ge 20\ell\Delta$. To bring this overhead down to a constant factor, we repeat the argument in Lines~\ref{line:second_while_start} to~\ref{line:second_while_end}, but set $\ell = e$.

By our parameter settings Proposition~\ref{prop:iterative-step} drops $\le \frac {\alpha n} {2K}$ rows and columns in each iteration, giving the lower bounds on $|S|, |T|$. Further, we can inductively apply Proposition~\ref{prop:iterative-step} to maintain that the rank $r$ of our iterate is bounded by $3^{2K}\rs = \rs \textup{poly}(\ell)$, since the potential function $r + \rs$ at most triples each iteration. The bounds on the runtime and $p$ then follow by using Proposition~\ref{prop:iterative-step} $2K$ times; we recall that we can aggregate the observation probabilities using Lemma~\ref{lem:multiple_obs}.
\end{proof}

To briefly interpret Corollary~\ref{cor:pmc}, let $\alpha = \frac 1 {200}$, in other words, consider the case where we are willing to give up on recovering a small constant fraction of rows and columns. Further, suppose $\frac{\sigma \sqrt{\rs}}{\Delta}$ is polynomially bounded in $m$, i.e.\ our initial distance estimate is not too far off from our noise level. By balancing terms via setting
\[\ell = \exp\Par{\sqrt{\log \Par{\frac{\sigma \sqrt{\rs}}{\Delta}}}} = \exp\Par{O\Par{\sqrt{\log(m)}}} = m^{o(1)},\]
we see that Corollary~\ref{cor:pmc} yields partial matrix completion obtaining the desired noise level $\Delta$ up to constant overhead, on at least $99\%$ of rows and columns. This setting of $\ell$ also implies that the rank of the iterates of Algorithm~\ref{alg:pmc} is bounded by $\rs \cdot m^{o(1)}$ throughout. Further, assuming $\delta = \text{poly}(m^{-1})$ for simplicity, the sample complexity of $mnp = O(m^{1 + o(1)}\rs)$ is almost information-theoretically optimal, and the runtime of $O(m^{1 + o(1)}(\rs)^2)$ is almost-verification time. In other words, by giving up on recovering a subconstant fraction of rows and columns, we obtain almost-optimal matrix completion on the remaining submatrix, without any subspace regularity assumptions.
\section{Recovering dropped subsets}\label{sec:fixing}
In this section, we provide the second key ingredient of our framework, an algorithm which recovers rows and columns which were dropped by our iterative method in Section~\ref{sec:partial}. The method of this section takes as input $\mm$ which satisfies submatrix closeness to our target rank-$\rs$ $\mms$, and returns a rank-$O(\rs)$ factorization. This factorization has the appealing property that it satisfies standard Frobenius norm closeness to $\mms$ (without any dropped rows or columns), at a cost of a roughly $\text{poly}(\rs)$ factor increase in the closeness bound. We now state the main export of this section, parameterized by the notion of standard subspaces (Definition~\ref{def:standard-subspace}).

\begin{restatable}{proposition}{restatefix}\label{prop:fix}
Let $\mms \in \R^{m \times n}$ be rank-$\rs$ with $(\alpha, \beta,\mu)$-standard row and column spans, $m \ge n$, $\delta \in (0, 1)$ and let $S \subseteq [m]$, $T \subseteq [n]$ have $|S| \ge m - \frac {\alpha n} 9$, $|T| \ge m - \frac {\alpha n} 9$. 
Assume $\mm \in \R^{m \times n}$ is given as a rank-$r$ factorization, $r \ge \rs$, $\mm_{S, T}$ is $(\frac \alpha {1800\log(m)}, \Delta)$-close to $\mms_{S, T}$ on a $\gamma$-submatrix, and $\mmh = \mms + \mn$ for $\normop{\mms} \le \sigma$, $\normf{\mn} \le \frac \Delta {20}$. Algorithm~\ref{alg:fix} returns $\mmu \in \R^{m \times r'}$ and $\mv \in \R^{n \times r'}$ satisfying 
\begin{equation}\label{eq:fix_bound}\normf{\mmu \mv^\top - \mms} \le \frac{\Cfix\rs\sqrt{\rs\log(\rs)}}{\beta^8}\Delta \text{ and } r' \le 2\rs,\end{equation}
for a universal constant $\Cfix$, with probability $\ge 1 - \delta$.
Algorithm~\ref{alg:fix} uses $O(\frac{mr^2\mu^2}{\alpha\beta^4}\log^2(\frac m {\beta\delta})\log(\frac{m(\sigma+\Delta)}{\Delta\beta\delta}))$ 
time and one call to $\oracle_p(\mmh)$ where for a sufficiently large constant,
\[p = O\Par{\frac{r\mu\log^2(\frac m \beta) \log(\frac m \delta)}{\alpha\beta^2 n}}.\]
\end{restatable}
 
We prove Proposition~\ref{prop:fix} in a number of steps organized into subsections, summarized as follows.

\begin{enumerate}
    \item In Section~\ref{ssec:sparsify_errors}, we give an algorithm $\Sparsify$ which takes matrices which are close on a submatrix and drops a few more rows and columns, with the guarantee that the resulting submatrices satisfy Definition~\ref{def:sparse-close}, i.e., they are close away from an RCS matrix.
    \item In Section~\ref{ssec:goodrowcol}, we give an algorithm $\Representative$ which takes as input a matrix which is close to the target away from an RCS matrix. By repeatedly testing for regression error of our iterate's columns against itself, $\Representative$ learns a subset $B$ of columns which are representative of the difference between our iterate and the target, in the sense of Definition~\ref{def:representative-subset}. 
    \item In Section~\ref{ssec:fill}, we first show that Definition~\ref{def:representative-subset} implies that the columns indexed by $B$ can be used to effectively approximate dropped rows and columns from observations. We then give an algorithm $\Complete$ which learns a low-rank approximation of $\mms$ to slightly higher error using $B$.
    \item We put all the pieces together to prove Proposition~\ref{prop:fix} in Section~\ref{ssec:fix}.
\end{enumerate}

\subsection{Sparsifying errors}\label{ssec:sparsify_errors}

The output of the method of Section~\ref{sec:partial} is $\tmm$ that is $\Delta$-close to the true matrix $\mms$ on a $\gamma$-submatrix (up to dropped subsets).  We begin with a postprocessing step which yields finer control over the structure of $\tmm - \mms$.  In particular, we drop some additional rows and columns so that the difference $\tmm - \mms$ is close away from an $s$-RCS matrix, for $s \approx \frac n r$. The algorithm is Algorithm~\ref{alg:sparsify}, and its statement and analysis are similar to that of Algorithm~\ref{alg:filter}. 

We use the following concentration inequality to control the error of empirical estimates.

\begin{lemma}\label{lem:empirical_large}
Let $p, \delta \in (0, 1)$, $\tau > 0$, $v \in \R^d$, and let $\tv \in \R^d$ have each entry $\tv_i$ independently set to $v_i$ with probability $p$, and $0$ otherwise. Then with probability $\ge 1 - \delta$,
\[\Abs{\Abs{\{i \in d \mid |v_i| \ge \tau\}} - \frac 1 p \Abs{\{i \in d \mid |\tv_i| \ge \tau\}}} \le \max\Par{\frac 1 {10} \Abs{\{i \in d \mid |v_i| \ge \tau\}},\; \frac{30\log \frac 2 \delta}{p}}.\]
\end{lemma}
\begin{proof}
Let $x_i \in \{0, 1\}$ be a scalar random variable for all $i \in [d]$ which is $1$ if $|v_i| \ge \tau$ and let $\tx_i$ be analogously defined for $\tv$; clearly $\E \tx_i = px_i$. Fact~\ref{fact:matchern} shows that with probability $\ge 1 - \delta$,
\[\Abs{\sum_{i \in [d]} x_i - \frac 1 p \sum_{i \in [d]} \tx_i} \le \sqrt{\frac{\sum_{i \in [d]} x_i} p} \sqrt{3\log \frac 2 \delta}. \]
The conclusion follows depending on which of $\frac{\sqrt{\sum_{i \in [d]} x_i}}{\sqrt{10}}$ or $\frac 1 {\sqrt p}\sqrt{30\log \frac 2 \delta} $ is larger.
\end{proof}

We now use an analogous argument to that of Lemma~\ref{lem:filter-step} to analyze Algorithm~\ref{alg:sparsify}.

\begin{lemma}\label{lem:postprocess}
Let $0 \le \tau \le \Delta$ and $\gamma, \gdrop, p, \delta \in (0, 1)$, and $s \in [n]$. Assume $\mm \in \R^{m \times n}$ is $\Delta$-close to $\mmh$ on a $\gamma$-submatrix, and that $m \ge n$. Finally define $s \defeq \frac{16\Delta^2}{\tau^2 \gdrop n}$ and assume
\[\gdrop \ge 200\gamma\log(m),\; p \ge \frac{20\gdrop n \tau^2}{\Delta^2} \log\Par{\frac{100m}{\delta}}.\]
With probability $\ge 1 - \delta$, Algorithm~\ref{alg:sparsify} returns $S \subseteq [m]$, $T \subseteq [n]$ satisfying the following.
\begin{enumerate}
    \item $|S| \geq m - \gdrop n$, $|T| \geq (1 - \gdrop)n$.
    \item $\mm_{S, T}$, $\mmh_{S, T}$ are $2\Delta$-close away from an $s$-RCS matrix.
\end{enumerate}
\end{lemma}
\begin{proof}
Our first goal is to show that before the application of $\Filter$, every row and column of $[\mm - \mmh]_{S, T}$ has at most $s$ entries larger than $\tau$ in magnitude. We will denote
\[\mds_t \defeq \Brack{\mm - \mmh}_{S_t \times T_t} \text{ and } \Phi_t \defeq \Abs{\Brace{(i, j) \in S_t \times T_t \bigm| |\Brack{\mds_t}_{ij}| \ge \tau }}.\]
In other words, $\Phi_t $ tracks the number of large entries in $\mds_t$, and by definition $\Phi_0 \le mn$ and $\Phi_t$ is nonincreasing. We also denote the exact number of large entries per row and column by
\begin{align*}
r_{i, t}^\star \defeq \Abs{\{j \in T_t \mid |[\mds_t]_{ij}| \ge \tau\}} \text{ for all } i \in S_t, \\
c_{j, t}^\star \defeq \Abs{\{i \in S_t \mid |[\mds_t]_{ij}| \ge \tau\}} \text{ for all } j \in T_t.
\end{align*}
Also, as in the proof of Lemma~\ref{lem:filter-step}, by applying Lemma~\ref{lem:empirical_large} and a union bound over $\le 40\log(m)$ iterations (giving the failure probability with a union bound over the call to $\Filter$ succeeding), we assume that for all $0 \le t < t_{\max}$ and all $i \in S_t$, $j \in T_t$,
\begin{align*}
\Abs{r_{i, t} - r^\star_{i, t}} &\le \max\Par{\frac 1 {10} r^\star_{i, t},\; \frac{\Delta^2}{2\tau^2} \cdot \frac{1}{\gamma n}},\\
\Abs{c_{j, t} - c^\star_{j, t}} &\le \max\Par{\frac 1 {10} c^\star_{j, t},\; \frac{\Delta^2}{2\tau^2} \cdot \frac{1}{\gamma n}},
\end{align*}
as well as, for all $i \in S_{t_{\max}}, j \in T_{t_{\max}}$,
\begin{align*}
\Abs{r_{i} - r^\star_{i, t_{\max}}} &\le \max\Par{\frac 1 {10} r^\star_{i, t_{\max}},\; \frac s {10}},\\
\Abs{c_j - c^\star_{j, t_{\max}}} &\le \max\Par{\frac 1 {10} c^\star_{j, t_{\max}},\; \frac s {10}}.
\end{align*}
We observe that two matrices which are $\Delta$-close in Frobenius norm have at most $\frac{\Delta^2}{\tau^2}$ entries of the difference with magnitude more than $\tau$. Next, consider an iteration $t$ where $\Phi_t \ge \frac{4\Delta^2}{\tau^2}$. By an analogous argument to Lemma~\ref{lem:filter-step}, removing the $\gamma n$ largest rows and columns decreases the potential by at least a $0.05$ factor, so inducting shows that after $t_{\max}$ iterations, 
\[\Phi_{t_{\max}} \le \frac{4\Delta^2}{\tau^2}.\]
Therefore by Markov's inequality there are at most $\frac{\gdrop n}{4}$ rows in $S_{t_{\max}}$ and $\frac{\gdrop n}{4}$ columns in $T_{t_{\max}}$ with at least $\frac{16\Delta^2}{\tau^2 \gdrop n} \le \frac s 2$ entries larger than $\tau$. In conclusion, if a row or column has more than $s$ entries larger than $\tau$ in the last iteration, it will be removed as claimed. 

In the last iteration $\mm_{S, T}$ and $\mmh_{S, T}$ are clearly still $\Delta$-close on a $\gamma$-submatrix, since we only dropped rows or columns. $\Filter$ ensures that by dropping $\frac{\gdrop n} 2$ more rows and columns, the difference matrix truncated at $\tau$ has Frobenius norm $\le 2\Delta$ (see the third guarantee of Lemma~\ref{lem:filter-step}). Hence, we can take $\mx$ to be the truncated difference and $\my$ to be the sparse errors in Definition~\ref{def:sparse-close}.
\end{proof}

\begin{algorithm2e}[ht!]\label{alg:sparsify}
\caption{$\Sparsify(\orzo(\mmh), \mm, \tau, \Delta, \gamma, \gdrop, p, \delta)$}
\DontPrintSemicolon
\codeInput $\orzo(\mmh)$, $\mm \in \R^{m \times n}$, $\tau, \Delta \ge 0$, $\gamma, \gdrop, p, \delta \in (0, 1)$\;
$S_0 \gets [m], T_0 \gets [n]$ \;
$t_{\max} \gets \lceil 20\log \frac{mn\tau^2}{4\Delta^2}\rceil$ \;
\For {$0 \le t < t_{\max}$}{
$\md_t \gets \oracle_p([\mm - \mmh]_{S_t, T_t})$ \;
\lFor{$i \in S_t$}{$r_{i,t} \gets \frac 1 p |\{j \in T_t \mid |[\md_t]_{ij}| \ge \tau\}|$ 
}
\lFor{$j \in T_t$}{$c_{j, t} \gets \frac 1 p |\{i \in S_t \mid |[\md_t]_{ij}| \ge \tau\}|$}
$S_{t+1} \gets S_t \setminus A_t$ where $A_t \subset S_t$ corresponds to the $\gamma n$ indices $i$ with largest $r_{i, t}$ \;
$T_{t + 1} \gets T_t \setminus B_t$ where $B_t \subset T_t$ corresponds to the $\gamma n$ indices $j$ with largest $c_{j, t}$  \;
}
$\md \gets \oracle_p([\mm - \mmh]_{S_{t_{\max}}, T_{t_{\max}}})$ \;
\lFor{$i \in S_{t_{\max}}$}{$r_{i} \gets \frac 1 p |\{j \in T_{t_{\max}} \mid |\md_{ij}| \ge \tau\}|$ 
}
\lFor{$j \in T_{t_{\max}}$}{$c_j \gets \frac 1 p |\{j \in T_{t_{\max}} \mid |\md_{ij}| \ge \tau\}$| }
$S \gets S_{t_{\max}} \setminus A$ where $A \subset S_{t_{\max}}$ corresponds to the $\frac{\gdrop n}{4}$ indices $i$ with largest $r_{i}$ \;
$T \gets T_{t_{\max}} \setminus B$ where $B \subset T_{t_{\max}}$ corresponds to the $\frac{\gdrop n}{4}$ indices $j$ with largest $c_{j}$  \;
\Return{$\Filter(\orzo(\mmh_{S, T}), \mm_{S, T}, \tau, \infty, \Delta, \gamma, \gdrop, 1, p, \frac \delta 2)$} \;
\end{algorithm2e}

\subsection{Learning a representative subset}\label{ssec:goodrowcol}

\subsubsection{Structural properties}

In this section, we collect several structural tools which will be helpful in the analysis of our testers. We first provide simple spectral bounds on a randomly subsampled matrix. 

\begin{lemma}\label{lem:regular-lower-bound}
Let $\delta \in (0, 1)$ and let $V \subseteq \R^d$ be $(\alpha, \beta,\mu)$-standard of dimension $r$, let $\{b_i\}_{i \in [d]} \subset \R^r$ be rows of an (arbitrary) choice of $\mb_V$, and let $S \subset [d]$ have $|S| \ge (1 - \frac \alpha 2)d$. Let $T \subseteq S$ have each element in $S$ included with probability $p \ge \frac{12 \mu r}{ \beta^2d}\log (\frac {2r} \delta)$. Then with probability $\ge 1 - \delta$,
\[\frac{p\beta^2} 2 \id_r \preceq \sum_{i \in T} b_ib_i^\top \preceq 2p\id_r.\]
\end{lemma}
\begin{proof}
By the assumption on the subspace $V$, there is a set $A \subset S$ of size at most $\frac {\alpha d}{3}$ such that every row $i \in R \defeq S \setminus A$ satisfies $\norm{b_i}_2 \le \sqrt{\frac{\mu r}{ d}}$.
For all $i \in R$ define a random matrix 
\[\mx_i \defeq \begin{cases}
b_i b_i^\top & \text{with probability }p, \\
\mzero_{r \times r} & \text{otherwise.}
\end{cases}\]
We recognize $\sum_{i \in R} \mx_i \preceq \sum_{i \in T} b_i b_i^\top$ (since it is restricted to a subset of the rows), and $\normop{\mx_i} \le \frac{\mu r}{d}$ with probability $1$ for all $i \in R$. Moreover, we have $\E \sum_{i \in R} X_i = p\sum_{i \in R} b_ib_i^\top$, and
\[p\beta^2 \id_r \preceq p\sum_{i \in R} b_ib_i^\top \preceq p \id_r,\]
by Lemma~\ref{lem:regular-wc} since $|R| \ge (1 - \alpha) d$. The conclusion follows by applying Fact~\ref{fact:matchern} with $\eps = \half$.
\end{proof}

We also use Proposition~\ref{prop:exists-lower-bound}, an existential variant of Lemma~\ref{lem:regular-lower-bound} which does not impose a regularity constraint, which can be viewed as a one-sided discrepancy statement potentially of independent interest. We use this terminology because a random $S$ of size $d\lam$ yields $\sum_{i \in S} b_ib_i^\top = \lam \id_r$ in expectation, and Proposition~\ref{prop:exists-lower-bound} matches this up to constant factors in the smallest eigenvalue (a one-sided guarantee). We defer a proof of this claim to Appendix~\ref{app:deferred-fixing}.

\begin{restatable}{proposition}{restateexistslowerbound}\label{prop:exists-lower-bound}
Let $\lam \in [\frac{5600r}{d}, 1)$, let $\mb \in \R^{d \times r}$ have orthonormal columns, and denote rows of $\mb$ by $\{b_i\}_{i \in [d]} \subset \R^r$. There exists $S \subseteq [d]$ with $|S| \le d\lam$ and
\[\sum_{i \in S} b_ib_i^\top \succeq \frac \lam 8 \id_r.\]
\end{restatable}

By using Proposition~\ref{prop:exists-lower-bound}, we show that removing a few columns from a pair of matrices which are close away from a RCS matrix induces a submatrix on which they are truly close.


\begin{lemma}\label{lem:most-rows-good}
Let $\mm, \mm' \in \R^{m \times n}$ and suppose $\mm, \mm'$ are $\Delta$-close away from an $s$-RCS matrix. Further, suppose $\mm - \mm'$ is rank-$r$. 
There is $T \subseteq [n]$ with $|T| \ge n - 5600rs\log m$ such that 
\[\normf{\Brack{\mm - \mm'}_{:T}} \le \Delta \cdot \frac{\sqrt{\log m}}{13}.\] 
\end{lemma}
\begin{proof}
Let $\md \defeq \mm - \mm'$ for notational convenience, and let $\md = \mx + \my$ be the decomposition guaranteed by Definition~\ref{def:sparse-close}.  Partition $[m]$ into sets $\{S_j\}_{j \in [k]}$ where $k \leq \log m$ as follows.  Let $S_1$ be the set of $i \in [m]$ such that $\norm{\mx_{i:}}_2 \leq \frac{2\Delta}{\sqrt m}$ and for $j >1$, let  $S_j$ be the set of $i \in [m]$ with
\[2^{j-1} \frac{\Delta}{\sqrt m} < \norm{\mx_{i:}}_2 \leq 2^j \frac{\Delta}{\sqrt m}.\]
Now for each $j \in [k]$, consider an SVD of $ \md_{S_j:} = \mmu_j \msig_j \mv_j^\top$ and apply Proposition~\ref{prop:exists-lower-bound} to $\mmu_j$ with $\lambda_j = \frac{5600r}{|S_j|}$.  We obtain $A_j \subseteq S_j$ such that $|A_j| \leq 5600r$ and for all $v \in \R^n$,
\begin{equation}\label{eq:Asuffices}
\norm{\md_{A_j:}v}_2 \geq \sqrt{\frac{\lambda_j} 8}\norm{\md_{S_j :}v}_2.
\end{equation}
Next recall $\md = \mx + \my$ where $\my$ is $s$-RCS.  For each $i \in [m]$, let $T_i \subset [n]$ be the set of entries on which $\my_{i:}$ is supported.  Let $T = [n] \setminus \bigcup_{i \in A_1 \cup \ldots \cup A_k} T_i$ so $|T| \ge n - 5600rs\log m$.  Now we can bound $\md_{:T}$ by applying \eqref{eq:Asuffices}:
\begin{align*}
\normf{\md_{:T}}^2 &= \sum_{j \in [k]} \normf{\md_{S_j, T}}^2 \le  \sum_{j \in [k]} \frac{8}{\lam_j}\normf{\md_{A_j, T}}^2 \\
&\le \sum_{j \in [k]} \frac{|S_j| |A_j|}{700r} \cdot \frac{2^{2j}\Delta^2}{m} \le \frac{\Delta^2 k}{175} \leq \frac{\Delta^2 \log m}{175},
\end{align*}
where in the second-to-last step, we used that $|S_j| \leq \frac{4m}{2^{2j}}$ by Markov's inequality.
\end{proof}

We note that Lemma~\ref{lem:most-rows-good} is a robust variant of a simpler claim, which says that a low-rank matrix with a sparse nonzero pattern must have all of its entries localized to a small submatrix. We provide a proof of this claim for convenience, as we believe it aids in building intuition for our method.

\begin{lemma}\label{lem:low-rank-sparse}
Let $\md \in \R^{m \times n}$ be rank-$r$ with $m \ge n$, and suppose $\md$ is $\frac{\alpha}{r}$-RCS for $\alpha \in (0, 1)$. There are $A \subseteq [m]$, $B \subseteq [n]$ with $|A| \ge m - \alpha n$, $|B| \ge (1 - \alpha)n$ such that $\md_{A, B}$ has no nonzero entries.
\end{lemma}
\begin{proof}
Consider an iterative process which takes any row of $\md$ with nonzero entries, and orthogonalizes all rows of $\md$ against it. The process terminates after $r$ iterations as $\md$ is rank-$r$, and the union of the supports of all rows used by the process grows by $\le \frac{\alpha n} r$ in each iteration. Hence, the support of all rows is contained in a subset of size $\alpha n$, and a symmetric argument holds for columns.
\end{proof}

\subsubsection{Basic testing}

We next analyze properties of a simple algorithm, $\Test$, which solves a regression problem attempting to boundedly combine a set of columns of a matrix to approximate another column. For ease of discussion, we focus on testing columns rather than rows, but a symmetric argument handles both.

\begin{algorithm2e}[ht!]\label{alg:test}
\caption{$\Test(\mm, T, j, \phi, \tau)$}
\DontPrintSemicolon
\codeInput $\mm \in \R^{m \times n}$, $T \subseteq [n]$, $j \in [n]$, $\phi, \tau \ge 0$\;
\lIf{$\min_{v \in \R^T } \norm{\mm_{:T} v - \mm_{:j}}_2^2 + \frac{\phi^2}{\tau^2}\norm{v}_2^2 \le 2\phi^2$}{
\Return{$\textup{``True''}$}
}
\lElse{
\Return{$\textup{``False''}$}
}
\end{algorithm2e}

If $\Test$ returns ``True'' then we say it has passed, and otherwise we say it has failed. Intuitively, $\Test$ simulates testing the value of the following constrained problem:
\[\min_{\substack{v \in \R^T \\ \norm{v}_2 \le \tau} } \norm{\mm_{:T} v - \mm_{:j}}_2 \le \phi,\]
but is easier to compute.
We use the following helper claim, which follows from a calculation.

\begin{fact}\label{fact:bounded_lincomb}
Let $\mm \in \R^{m \times n}$ have SVD $\mmu \msig \mv^\top$ and have rank $r$, and let $T \subseteq [n]$ satisfy $\mv_{T:}^\top \mv_{T:} \succeq \gamma^2 \id_r$. Then for some $j \in [n]$, letting $c \in \R^T$ be the vector such that $\mm_{:T} c = \mm_{:j}$, $\norm{c}_2 \le \frac 1 \gamma \norm{\mv_{j:}}_2$.
\end{fact}

Fact~\ref{fact:bounded_lincomb} will be used to show the regression problems we encounter have bounded solutions. Motivated by this, we next specify a set of properties that guarantee $\Test$ will pass, combining regularity of a comparison matrix $[\mms]_{:T}$ in the sense of Fact~\ref{fact:bounded_lincomb} with closeness of $\mm$, $\mms$.

\begin{definition}\label{def:representative-subset}
We say $T \subseteq [n]$ is a $(\Delta,\gamma)$-representative subset with respect to a pair of matrices $\mm, \mms \in \R^{m \times n}$ if the following properties hold.
\begin{itemize}
    \item $\normf{[\mm - \mms]_{:T}} \leq \Delta$.
    \item $[\mvs]_{T:}^\top [\mvs]_{T:} \succeq \gamma^2 \id_{\rs}$, where $\mmus \msigs \mvs^\top$ is an SVD of $\mms$ which has rank $\rs$.
\end{itemize}
\end{definition}

\begin{lemma}\label{lem:test-passes}
Let $T \subseteq [n]$ contain a $(\frac{\phi}{2\tau}, \frac \theta \tau)$-representative subset with respect to $\mm, \mms \in \R^{m \times n}$, for some $\theta \ge 0$. Let $j \in [n]$ satisfy $\norm{[\mvs]_{j:}}_2 \le \theta$ and $\norm{[\mm - \mms]_{:j}}_2 \le \frac \phi 2$, where $\mmus \msigs \mvs$ is an SVD of $\mms$. Then $\Test(\mm, T, j, \phi, \tau)$ will pass.
\end{lemma}
\begin{proof}
By Fact~\ref{fact:bounded_lincomb} and the fact that $T$ is a representative subset with parameter $\frac \theta \tau$, there is a vector $c \in \R^T$ with $\norm{c}_2 \le \tau$ and $\mms_{:T} c = \mms_{:j}$. Using the other property of a representative subset shows
\begin{align*}
\norm{\mm_{:T} c - \mm_{:j}}_2 &\le \norm{\mm_{:T} c - \mms_{:T} c}_2 + \norm{\mms_{:T} c - \mm_{:j}}_2 \\
&\le \normf{[\mm - \mms]_{:T}} \norm{c}_2 + \norm{[\mm - \mms]_{:j}}_2 \le \frac{\phi}{2\tau} \cdot \tau + \frac{\phi}{2} = \phi.
\end{align*}
It is then straightforward to check that $c$ attains objective value $2\phi^2$ as desired.
\end{proof}

To complete our analysis of $\Test$ we further specify a set of properties that guarantees it will fail. Intuitively our conditions impose that for a small set of coordinates (which cannot significantly affect subspace regularity), the deviation from the underlying matrix $\mms$ on a particular column restricted to those coordinates is substantially larger than it should be. 

\begin{lemma}\label{lem:test-fails}
Let $\mm, \mms \in \R^{m \times n}$, $m \ge n$, and $R \subset S \subseteq [m]$, $T \subseteq [n]$, $j \in [n]$ satisfy the following.
\begin{enumerate}
    \item $\normf{[\mm - \mms]_{S, T}} \le \frac \phi \tau$.\label{item:fail-1}
    \item $\norm{[\mm - \mms]_{S \setminus R, j}}_2 \le \phi$.\label{item:fail-2}
    \item $\norm{[\mm - \mms]_{R, j}}_2 \ge \frac{7\phi} \beta$.\label{item:fail-3}
    \item $|S| \ge m - \frac {\alpha n} 2$ and $|R| \le \frac {\alpha n} 2 $.\label{item:fail-4}
\end{enumerate}
Then if the column span of $\mms$ is $(\alpha, \beta)$-regular, $\Test(\mm, T, j, \phi, \tau)$ will fail.
\end{lemma}
\begin{proof}
Assume for contradiction that $\Test$ passes, and let $v \in \R^T$ be the solution. This implies
\begin{equation}\label{eq:reg_implies}\norm{v}_2 \le \sqrt{2}\tau,\; \norm{\mm_{:T}v - \mm_{:j}}_2 \le \sqrt{2}\phi.\end{equation}
We next observe that
\begin{align*}
\norm{\mms_{S \setminus R, T}v - \mms_{S \setminus R, j}}_2 &\le \norm{\mms_{S \setminus R, T}v - \mm_{S \setminus R, j}}_2 + \phi \\
&\le \norm{[\mm - \mms]_{S \setminus R, T} v}_2 + \norm{\mm_{S \setminus R, T} v - \mm_{S \setminus R, j}}_2 + \phi \le 4\phi,
\end{align*}
where the first line used Item~\ref{item:fail-2}, and the second used Item~\ref{item:fail-1}, $\normop{\cdot} \le \normf{\cdot}$, the bounds in \eqref{eq:reg_implies}, and $1 + 2\sqrt{2} \le 4$. We further have
\begin{align*}
\norm{\mms_{R, T} v - \mms_{R, j}}_2 &\ge \frac{7\phi}{\beta} - \norm{\mms_{R, T} v - \mm_{R, j}}_2 \\
&\ge \frac{7\phi}{\beta} - \norm{[\mm - \mms]_{R, T} v}_2 - \norm{\mm_{R, T} v - \mm_{R, j}}_2 \ge \frac{4\phi}{\beta},
\end{align*}
where the first line used Item~\ref{item:fail-3}, and the second line followed similarly to the previous calculation. Finally, note that the column span of $\mms_{S:}$ is a $(\frac \alpha 2, \beta)$-regular subspace by the size bound on $S$ from Item~\ref{item:fail-4}, and $u \defeq \mms_{S, T} v - \mms_{S, j}$ is an element of this subspace. However, we have proven $\norm{u_{S \setminus R}}_2 \le \beta \norm{u}_2$ by combining the above displays, a contradiction to Lemma~\ref{lem:regular-wc}.
\end{proof}

We now give a consequence of Lemma~\ref{lem:test-fails} that handles the case of a randomly chosen subset $T$.

\begin{lemma}\label{lem:test-guarantees}
Let $\mm, \mms \in \R^{m \times n}$ be $\Delta$-close away from an $s$-RCS matrix, and let $\mx + \my$ be the decomposition in Definition~\ref{def:sparse-close}. Let $T \subset [n]$ have each element included independently with probability $p$, and suppose
\[\norm{\my_{:j}}_2 \ge \frac{100}{\beta}\Par{\norm{\mx_{:j}}_2 + \tau\sqrt{p}\Delta + \phi},\]
for some $j \in [n]$. Finally suppose $s \le \frac{\alpha\min(m, n)}{2}$, and $p \le \frac{\alpha}{1000s}$. Then with probability at least $0.9$ over the randomness of $T$, $\Test(\mm, T, j, \phi, \tau)$ fails.
\end{lemma}
\begin{proof}
For all $k \in [n]$ let $S_k$ be the support of the $\my_{:k}$ satisfying $|S_k| \le s$, and let $S \defeq [m] \setminus \bigcup_{k \in T} S_k$. 
We first condition on the following three events, each of which happens with probability at least $0.99$ by Markov's inequality, giving the failure probability via a union bound.
\begin{enumerate}
    \item $\normf{\mx_{:T}} \le 10\sqrt{p}\Delta$.
    \item $|T| \le 100pn$.
    \item $\norm{\my_{S \cap S_j,j}}_2 \ge 0.9\norm{\my_{:j}}_2$.
\end{enumerate}
To see that the last event holds with probability $0.99$, we used Markov's inequality and
\begin{align*}
\E\Brack{\norm{\my_{S_j \setminus S, j}}_2^2} = \E\Brack{\norm{\my_{\bigcup_{k \in T} S_k, j}}_2^2} \le ps\norm{\my_{:j}}_2^2 \le \frac 1 {1000}\norm{\my_{:j}}_2^2,
\end{align*}
because for each $i \in S_j$, at most $s$ other columns of $\my$ have $i \in S_k$ by assumption.
Under these events, we now prove $\Test$ fails by applying Lemma~\ref{lem:test-fails} with parameters $\phi', \tau$ where
\[\phi' \defeq 10\tau\sqrt{p}\Delta + \norm{\mx_{:j}}_2 + \phi.\]
We will use $R = S \cap S_j$, so Item~\ref{item:fail-4} of Lemma~\ref{lem:test-fails} follows from the assumed bound on $s \ge |R|$, and that $|[m] \setminus S| \le 100pns \le \frac{\alpha n}{2}$. Item~\ref{item:fail-1} follows because $[\mm - \mms]_{S, T} = \mx_{S, T}$ (as $\my_{S, T}$ is zero by definition), and $\normf{\mx_{S, T}} \le \normf{\mx_{:T}} \le \frac{\phi'}{\tau}$. Item~\ref{item:fail-2} follows because 
\[\norm{[\mm - \mms]_{S \setminus R, j}}_2 = \norm{\mx_{S\setminus R, j}}_2 \le \norm{\mx_{:j}}_2 \le \phi'.\]
Finally, Item~\ref{item:fail-3} follows because
\begin{align*}
\norm{\Brack{\mm - \mms}_{S \cap S_j, j}}_2 &\ge \norm{\my_{S \cap S_j, j}}_2 - \norm{\mx_{:j}}_2 \ge \frac{7\phi'}{\beta}.
\end{align*}
We note that as $\phi \le \phi'$, $\Test$ failing with parameter $\phi'$ implies $\Test$ with parameter $\phi$ also fails.
\end{proof}
\subsubsection{Finding a representative subset}

In this section, we finally analyze our main algorithm, $\Representative$, for finding a representative subset of columns of an iterate $\mm$ in the sense of Definition~\ref{def:representative-subset}, assuming $\mm$ is close to $\mms$ away from an RCS matrix. We showed in Lemma~\ref{lem:test-passes} that this ensures good regression error on completing other columns of $\mm$. This property will be used with the representative subset we return in the next Section~\ref{ssec:fill} to complete our current matrix (including rows and columns we dropped).

\begin{algorithm2e}\label{alg:row/column-selector}
\caption{$\Representative(\mm, \phi, p)$}
\DontPrintSemicolon
\codeInput $\mm \in \R^{m \times n}$, $\phi \ge 0$, $p\in (0, 1)$  \; 
Sample $B_0 \subseteq [n]$  by independently including each $k \in [n]$ with probability $p$ \;
$\Count \gets \0_{B_0}$\;
$t_{\max} \gets \lceil 40 \log(mn) \rceil$\;
$\tau \gets \frac{1}{\sqrt{40\log(r)}}$\;
\For{$t \in [t_{\max}]$ }{
Sample $T \subseteq [n]$ by independently including each $k \in [n]$ with probability $p$ \;
\For{$j \in B_0$}{
\lIf{$\Test(\mm, T, j, \phi, \tau)$}{
$\Count_j \gets \Count_j + 1$
}
}
}
$B \gets B_0$ with all $j \in B_0$ satisfying $\Count_j \leq \half t_{\max}$ removed \;
\Return{$B$}\;
\end{algorithm2e}

\begin{lemma}\label{lem:column-tester}
Let $\mm \in \R^{m \times n}$ be given as a rank-$r$ factorization and $\Delta$-close to $\mms \in \R^{m \times n}$ away from an $s$-RCS matrix, where $\mms$ is rank-$\rs$ with $(\alpha,\beta,\mu)$-standard row and column spans. If $m \ge n$, $r \ge \rs$, and
\[s \le \min\Par{\frac{\alpha \min(m, n)}{10^5 r\log (mn)},\frac{\alpha}{1000p}},\; \phi = \Delta \sqrt{p\log m},\; p \ge \frac{40\mu \rs}{\beta^2 n}\log\Par{\rs},\]
then letting $B \subseteq [n]$ be the subset output by Algorithm~\ref{alg:row/column-selector} with parameters $(\mm, \phi, p)$, $B$ is a $(\tDelta, \gamma)$-representative subset with respect to $\mm, \mms$ with probability at least $0.9$, for
\[\tDelta \defeq \Delta \cdot 1000p\sqrt{\frac{n}{\beta^2}} ,\; \gamma \defeq \sqrt{\frac{p\beta^2}{2}}.\]
\end{lemma}
\begin{proof}
Throughout we follow the parameter settings of $\tau$ and $t_{\max}$ in Algorithm~\ref{alg:row/column-selector}. We begin by proving the second condition in Definition~\ref{def:representative-subset}. By Lemma~\ref{lem:most-rows-good} and the upper bound on $s$, there is a subset $Q \subseteq [n]$ with $|Q| \ge (1 -\frac {\alpha} {10})n$ such that
\begin{equation}\label{eq:qdef}\normf{\Brack{\mm - \mms}_{:Q}} \le \Delta \cdot \frac{\sqrt{\log m}}{13}.\end{equation}
Further, let $\mmus \msigs \mvs^\top$ be an SVD of $\mms$, and let $Q' \subseteq Q$ be the indices $j$ satisfying both
\begin{align*}
\norm{[\mm - \mms]_{:j}}_2 \le \Delta \cdot \sqrt{\frac{\log m}{\alpha n}} \le \frac \phi 2,\; \norm{\Brack{\mvs}_{j:}}_2 \le \theta \defeq \sqrt{\frac{2\mu r}{n}}.
\end{align*}
By Markov's inequality and the definition of a regular subspace, we have that $|Q'| \ge (1 - \frac \alpha 2)n$. We next claim that every index in $Q' \cap B_0$ will be included in $B$ with probability at least $0.99$. It suffices to prove that the conditions of Lemma~\ref{lem:test-passes} are met with probability at least $0.9$, and then a Chernoff bound shows a majority of the $t_{\max}$ tests will pass with probability $\ge 1 - \frac 1 {100n}$ for each $j \in Q' \cap B_0$. To see this, by Markov's inequality and \eqref{eq:qdef}, with probability at least $0.99$ we have
\[\normf{\Brack{\mm - \mms}_{:T\cap Q}} \le \Delta \sqrt{p\log m} \le \frac{\phi}{2\tau},\]
and by Lemma~\ref{lem:regular-lower-bound} with $S \gets Q$, with probability at least $0.99$,
\[\Brack{\mvs}_{T \cap Q:}^\top\Brack{\mvs}_{T \cap Q:} \succeq \frac{p\beta^2}{2}\id_{\rs} \succeq \frac{\theta^2}{\tau^2}\id_{\rs}.\]
Clearly $T$ contains $T \cap Q$ so the conditions of Lemma~\ref{lem:test-passes} are all met with probability $0.9$. Therefore, conditioning that $B \supseteq Q' \cap B_0$, and since $B_0$ is independently sampled from $Q'$, applying Lemma~\ref{lem:regular-lower-bound} once more with $S \gets Q'$ shows the first condition of Definition~\ref{def:representative-subset} is met with probability $0.95$.

Now we verify the first of the desired conditions in Definition~\ref{def:representative-subset}. Let $\mm - \mms = \mx + \my$ be the promised decomposition from Definition~\ref{def:sparse-close}, and note that the given bound on $s$ shows the preconditions of Lemma~\ref{lem:test-guarantees} are met. Therefore for every $j \in B$, a Chernoff bound shows that with probability $0.99$, the contrapositive of Lemma~\ref{lem:test-guarantees} holds, i.e.\
\[
\norm{\my_{:j}}_2^2 \le \frac {30000} {\beta^2} \Par{\norm{\mx_{:j}}_2^2 + \tau^2 p \Delta^2 + \phi^2} \le \frac {30000} {\beta^2} \Par{\norm{\mx_{:j}}_2^2 + 1.5\phi^2},
\]
where we used $(a + b + c)^2 \le 3(a^2 + b^2 + c^2)$ and the lower bound on $\phi$. Summing the above display over $j \in B$ and using $\norm{u + v}_2^2 \le 2\norm{u}_2^2 + 2\norm{v}_2^2$ with $u \gets \mx_{:j}$ and $v \gets \my_{:j}$, we have
\begin{equation}\label{eq:suffice_x_small}
\normf{\Brack{\mm - \mms}_{:B}}^2 \le \frac{90000 pn \phi^2}{\beta^2} + \frac{60000pn}{\beta^2}\normf{\mx_{:B}}^2,
\end{equation}
since $|B| \le |B_0| \le 2pn$ with probability at least $0.99$ by a Chernoff bound. Finally, the conclusion follows since $\normf{\mx_{:B}}^2 \le 30p\Delta^2$ with probability at least $0.97$ by Markov's inequality, and we union bound over these two events and the prior failure probabilities.
\end{proof}

\subsection{Filling in the matrix}\label{ssec:fill}

\subsubsection{Completing columns with a representative subset}

In this section, we show that given a representative subset of columns (in the setting of Lemma~\ref{lem:column-tester}), we can efficiently learn coefficients completing the rest of our iterate $\mm$ as combinations of the subset via observations from $\mms$. We begin by proving several helper regularity bounds which will allow us to argue that the regression problems we solve are well-conditioned with good probability. Specifically, we analyze the regularity of a (truncated) span of our representative columns.

\begin{lemma}\label{lem:svd-perturbation-bound}
Let $\ma \in \R^{m \times n}$ be rank-$r$ with SVD $\mmu \msig \mv^\top$, and let $\mb \in \R^{m \times n}$ satisfy $\normf{\ma - \mb} \le \Delta$. For some $\theta \in (0, 1)$ let $\mb'$ be the matrix obtained by taking an SVD of $\mb$ and dropping singular values smaller than $\frac \Delta \theta$. Let $\mmu' \msig' (\mv')^\top$ be an SVD of $\mb'$. Then the following statements hold.
\begin{enumerate}
\item $\mmu'$ has rank at most $2r$.
\item $\norm{(\id_m - \mmu\mmu^\top)u}_2 \le \theta$ for all unit vectors $u$ in the column span of $\mmu'$.
\item $\normf{\ma - \mb'} \le \frac{4\sqrt{r}\Delta}{\theta}$.
\end{enumerate}
\end{lemma}
\begin{proof}
To see the first claim, Lemma~\ref{lem:lowrank+short-matrix} (overloading the application with $\mb \gets \mb - \ma$) shows that $\mb$ has at most $2r$ singular values more than $\frac \Delta {\sqrt r}$, so $\mb'$ is rank at most $2r$. We move onto the second claim: let $u \in \R^m$ be in the column span of $\mmu'$. We bound
\begin{align*}
\norm{\mb^\top \Par{\id_m - \mmu \mmu^\top} u}_2 &\ge \norm{\mv' \msig' (\mmu')^\top\Par{\id_m - \mmu \mmu^\top } u}_2 \\
&\ge \frac{\Delta}{\theta} \norm{(\mmu')^\top\Par{\id_m - \mmu \mmu^\top } u}_2 \\
&\ge \frac{\Delta}{\theta} u^\top\Par{\id_m - \mmu \mmu^\top } u = \frac{\Delta}{\theta}\norm{\Par{\id_m - \mmu \mmu^\top } u}_2^2,
\end{align*}
where the first inequality followed since $\mb' = \mmu' \msig' (\mv')^\top$ drops singular values from $\mb$, the second used orthonormality of $\mv'$ and our lower bound on $\msig'$, the third used that $u$ is contained in the column span of $\mmu'$, and the last used that $\id_m - \mmu\mmu^\top$ is a projector. On the other hand,
\begin{align*}
\norm{\mb^\top\Par{\id_m - \mmu\mmu^\top} u}_2 &= \norm{\Par{\ma - \mb}^\top\Par{\id_m - \mmu\mmu^\top} u}_2 \le \Delta\norm{\Par{\id_m - \mmu\mmu^\top} u}_2
\end{align*}
where we used $\ma = \mmu\mmu^\top\ma$. The above two displays yield the second claim. To see the third,
\[\normop{\mb' - \ma} \le \normop{\mb' - \mb} + \normf{\mb - \ma} \le \frac{2\Delta}{\theta}. \]
Since $\mb' - \ma$ is rank at most $3r$, the conclusion follows from $\normf{\mb' - \ma} \le \sqrt{3r}\normop{\mb' - \ma}$.
\end{proof}

Applying Lemma~\ref{lem:svd-perturbation-bound} then yields a regularity bound on a truncated SVD of our iterate.

\begin{lemma}\label{lem:subspace-error-scaling}
Let $B \subseteq [n]$ be $(\tDelta, \gamma)$-representative with respect to $\mm, \mms \in \R^{m \times n}$, and assume $\mms$ is rank-$\rs$ with SVD $\mmus \msigs \mvs^\top$ and $(\alpha,\beta,\mu)$-standard row and column spans. Let $\mmu \msig \mv^\top$ be an SVD of $\mm_{:B}$ after dropping singular values smaller than $\frac{2\tDelta}{\beta}$. Then the following statements hold.
\begin{enumerate}
    \item $\mmu$ has rank at most $2\rs$.
    \item The column span of $\mmu$ is $(\alpha, \frac \beta 2)$-regular.
    \item There is a matrix $\my \in \R^{r \times n}$, where $\mmu \in \R^{m \times r}$, satisfying
    \[\normf{\mmu \my - \mms} \le \frac{8\sqrt{\rs}\tDelta}{\gamma\beta}.\]
\end{enumerate}
\end{lemma}
\begin{proof}
We are in the setting of Lemma~\ref{lem:svd-perturbation-bound} with $\ma \gets \mms_{:B}$, $\mb \gets \mm_{:B}$, and $\theta \gets \frac \beta 2$, so the first claim follows. The second claim in Lemma~\ref{lem:svd-perturbation-bound} shows any unit $u$ in the column span of $\mmu$ can be decomposed as $u = v + w$ where $v$ is the projection of $u$ into the column space of $\mmus$ and $\norm{w}_2 \le \frac \beta 2$. Since $v$
 is in the column span of $\mmus$, Lemma~\ref{lem:regular-wc} shows that for any $S \subseteq [m]$ with $|S| \ge (1-\alpha) m$,
 \begin{align*}
\norm{v_S}_2 \ge \beta\norm{v}_2 \implies \norm{u_S}_2 \ge \norm{v_S}_2 - \norm{w_S}_2 \ge \beta - \frac \beta 2 = \frac \beta 2,
 \end{align*}
 proving the desired regularity of the column span of $\mmu$ via Lemma~\ref{lem:regular-wc}. To see the last claim, representativeness of $B$ shows that by taking $\mz = [\mvs]_{B:} ([\mvs]_{B:}^\top[\mvs]_{B:})^{-1}\mvs^\top \in \R^{|B| \times n}$, 
 \[\normop{\mz} = \sqrt{\lam_1(\mz\mz^\top)} = \sqrt{\lam_1\Par{([\mvs]_{B:}^\top[\mvs]_{B:})^{-1}}} \le \frac 1 \gamma.\]
 Further, this $\mz$ satisfies $\mms = \mms_{:B}\mz$. Hence for $\my = \msig \mv^\top \mz$, we have the desired
 \begin{align*}
\normf{\mmu\my - \mms} &= \normf{\Par{\mmu\msig\mv^\top - \mms_{:B}} \mz} \le \frac 1 \gamma \normf{\mmu\msig\mv^\top - \mms_{:B}} \le \frac{8\sqrt{\rs}\tDelta}{\gamma\beta}.
 \end{align*}
 Above we used the last claim of Lemma~\ref{lem:svd-perturbation-bound} and, for any $\ma$ with $|B|$ columns,
 \[\normf{\ma\mz}^2 = \inprod{\ma^\top\ma}{\mz\mz^\top} \le \frac 1 {\gamma^2} \inprod{\ma^\top\ma}{\id_{|B|}} = \frac 1 {\gamma^2}\normf{\ma}^2.\]
\end{proof}

We further require one helper claim on regression error from noisy observations.

\begin{lemma}\label{lem:regression}
Let $v = \mmu y + \xi$ for $\mmu \in \R^{m \times r}$ with orthonormal columns. Suppose $\mmu_{A:}^\top \mmu_{A:} \succeq \lam^2 \id_r$ for $A \subseteq [m]$.
Then for $c^\star \defeq \arg\min_{c \in \R^r}\norm{\mmu_{A:} c - v_{A:}}_2$, and any $\hc \in \R^r$ with $\norm{\hc - c^\star}_2 \le \Delta$,
\[\norm{\mmu \hc - v}_2 \le \norm{\xi}_2 + \frac{2}{\lam} \norm{\xi_{A:}}_2 + \Delta.\]
\end{lemma}
\begin{proof}
Because setting $c = y$ attains error $\norm{\xi_{A:}}_2$, we must have $\norm{\mx_{A:} c^\star - v_{A:}}_2 \le \norm{\xi_{A:}}_2$. The conclusion follows from the assumption on $A$ and the triangle inequality:
\begin{align*}
\norm{\mmu \hc - v}_2 &\le \norm{\mmu y - v}_2 + \norm{\mmu (c^\star - y)}_2 + \norm{\mmu(c^\star - \hc)}_2 \\
&\le \norm{\xi}_2 + \frac 1 \lam \norm{\mmu_{A:}(\hc - y)}_2 + \Delta \\
&\le \norm{\xi}_2 + \frac 1 \lam \norm{\mmu_{A:} \hc - v_{A:}}_2 + \frac 1 \lam \norm{\mmu_{A:} y - v_{A:}}_2 + \Delta \le \norm{\xi}_2 + \frac{2}{\lam} \norm{\xi_{A:}}_2 + \Delta.
\end{align*}
\end{proof}

Finally, we state a standard result on the runtime of well-conditioned linear regression.

\begin{proposition}[\cite{Nesterov83}]\label{prop:agd}
Let $\ma \in \R^{d \times r}$ have full column rank, let $b \in \R^d$, and let 
\[x^\star \defeq \arg\min_{x \in \R^r} \norm{\ma x - b}_2^2.\]
There is an algorithm $\AGD(\ma, b, x_0, N)$ which outputs $x \in \R^r$ in time $O(\tmv(\ma)\cdot N)$ satisfying $\norm{x - x^\star}_2 \le \Delta$, if
\[N \ge \sqrt{\kappa(\ma^\top \ma)}\log\Par{\frac{2\kappa(\ma^\top\ma)\norm{x_0 - x^\star}_2^2}{\Delta^2}}.\]
\end{proposition}

We now analyze our subroutine for learning coefficients with respect to a representative subset.

\begin{algorithm2e}\label{alg:regression}
\caption{$\Complete(\oracle_p(\mmh), \mm_{:B}, \rs, B, \Delta, \tDelta, \sigma, \alpha, \beta)$}
\DontPrintSemicolon
\codeInput $\oracle_p(\mmh)$ for $p \in (0, 1)$ and $\mmh = \mms + \mn \in \R^{m \times n}$ where $\mms$ is rank-$\rs$, $\normf{\mn} \le \Delta$, $\mm_{:B} \in \R^{m \times |B|}$, $B \subseteq [n]$, $\tDelta, \sigma \ge 0$, $\alpha, \beta \in (0, 1)$ \; 
$\mmu \msig \mv^\top \gets $ SVD of $\mm_{:B}$ with singular values smaller than $\frac{2\tDelta}{\beta}$ dropped, for $\mmu \in \R^{m \times r'}$\;
$\hmv \gets \mzero_{n \times r'}$\;
\lIf{$r' > 2\rs$}{\Return{$(\mmu, \hmv)$}}\label{line:rank_too_large}
$R \gets \{i \in [m] \mid \norm{\mmu_{i:}}_2^2 \ge \frac{2r'}{\alpha n}\}$\;
$N \gets \lceil\frac 4 \beta \log(\frac{3\cdot 10^5 \rs(\Delta^2 + \tDelta^2 + \sigma^2)n}{p\gamma^2\beta^6\Delta^2})\rceil$\;
\lFor{$j \in [n]$}{$S_j \gets A_j \setminus R$ where $A_j \subseteq [m]$ corresponds to revealed entries of $\mmh_{:j}$}
\lFor{$j \in [n]$ }{$\hmv_{:j} \gets \AGD(\mmu_{S_j:}, \mmh_{S_j, j}, \0_{r'}, N)$ (see Proposition~\ref{prop:agd}) 
}
\Return{$(\mmu, \hmv)$}\;
\end{algorithm2e}

\begin{lemma}\label{lem:regression_error}
Following notation of Algorithm~\ref{alg:regression}, suppose $B$ is $(\tDelta, \gamma)$-representative with respect to $\mms, \mm \in \R^{m \times n}$, $\mms$ has $(\alpha,\beta)$-regular row and column spans, $\normop{\mms} \le \sigma$, $\normf{\mn} \le \Delta$, and $p \ge \frac{500\rs}{\alpha\beta^2 n}\log(n)$. Then with probability at least $0.9$ over the randomness of $\oracle_p(\mmh)$, 
\[\normf{\mmu \hmv^\top - \mms} \le \frac{200}{\beta^2} \cdot \Par{\Delta + \frac{\sqrt{\rs}\tDelta}{\gamma}}\text{ and } r' \le 2\rs .\]
\end{lemma}
\begin{proof}
By Lemma~\ref{lem:subspace-error-scaling}, whenever $B$ is $(\tDelta, \gamma)$-representative, the algorithm never terminates on Line~\ref{line:rank_too_large}, and the column space of $\mmu$ is $(\alpha, \frac \beta 2)$-regular (and hence $(\alpha, \beta,\frac 3 \alpha)$-standard).
We condition on the following two events, each of which holds with probability at least $0.95$, giving the failure probability by a union bound. First, by an application of Fact~\ref{fact:matchern} analogous to its use in proving Lemma~\ref{lem:regular-lower-bound}, since $|R| \le \frac{\alpha n}{2}$, we have for all $j \in [n]$ simultaneously,
\begin{equation}\label{eq:wc_random_sample}\frac{p\beta^2}{8} \id_{r'} \preceq \sum_{i \in S_j} u_i u_i^\top \preceq 2p\id_{r'}.\end{equation}
Second, let $\mn' = \mms - \mmu \my$ be the difference matrix from Lemma~\ref{lem:subspace-error-scaling}, so that $\mmh = \mmu \my + \mn + \mn'$ and $\mn'$ is independent of $\oracle_p(\mmh)$. We will condition on the following via Markov's inequality:
\begin{equation}\label{eq:sampled_noise}\sum_{j \in [n]} \norm{\Brack{\mn + \mn'}_{S_j, j}}_2^2 \le 20p \normf{\mn + \mn'}^2.\end{equation}
Under these events, Lemma~\ref{lem:subspace-error-scaling} also proves $\normf{\mn'} \le \frac{8\sqrt{\rs}\tDelta}{\gamma\beta}$, and by orthonormality of $\mmu$,
\begin{equation}\label{eq:ybound}\normf{\my} = \normf{\mmu\my} \le \normf{\mms} + \normf{\mn'} \le \sigma\sqrt{\rs} + \frac{8\sqrt{\rs}\tDelta}{\gamma\beta}.\end{equation}
Finally, for all $j \in [n]$ we bound the error of $\AGD$. Let $c^\star_j$ minimize $\norms{\mmu_{S_j:} c - \mmh_{S_j, j}}_2^2$, and for simplicity let $\ma_j \defeq \mmu_{S_j:}$ and $b_j \defeq \mmh_{S_j,j}$. By Lemma~\ref{lem:regression} with $y \gets \my_{:j}$ and $\xi \gets [\mn + \mn']_{:j}$, 
\begin{equation}\label{eq:opt_error_reg}\norm{\ma_j c^\star_j - b_j}_2^2 \le 2\norm{\Brack{\mn + \mn'}_{:j}}_2^2 + \frac {64}{p\beta^2}\norm{\Brack{\mn + \mn'}_{S_j,j}}_2^2,\end{equation}
where we used the lower bound $\lam^2 = \frac{p\beta^2}{8}$ in \eqref{eq:wc_random_sample}. Further, by integrating the lower bound in \eqref{eq:wc_random_sample},
\begin{align*}
\norm{c^\star_j - \my_{S_j, j}}_2^2 &\le \frac{8}{p\beta^2} \norm{\ma_j(c^\star_j - \my_{S_j, j})}_2^2 \\
&= \frac 8 {p\beta^2} \Par{\norm{\ma_j \my_{S_j, j} - b_j}_2^2 - \norm{\ma_j c^\star_j - b_j}_2^2} \\
&\le \frac{8}{p\beta^2} \normf{\mmu \my - \mmh}^2 \le \frac{8}{p\beta^2} \normf{\mn + \mn'}^2,
\end{align*}
so plugging in \eqref{eq:ybound} gives the crude bound
\begin{equation}\label{eq:initial_error_reg}
\begin{aligned}
\norm{c^\star_j}_2^2 &\le \frac{24}{p\beta^2}\normf{\mn + \mn'}^2 + 3\sigma^2\rs + \frac{192\rs\tDelta^2}{\gamma^2\beta^2} \le \frac{3300\rs}{p\gamma^2\beta^4}\Par{\Delta^2 + \tDelta^2 + \sigma^2}.
\end{aligned}
\end{equation}
Therefore, by combining \eqref{eq:opt_error_reg}, \eqref{eq:initial_error_reg}, the condition number bound in \eqref{eq:wc_random_sample}, and Proposition~\ref{prop:agd}, running for $N$ iterations yields $\hc_j \defeq \hmv_{:j}$ satisfying $
\norms{\hc_j - c^\star_j}_2 \le \frac{\Delta}{\sqrt{2n}}$,
so by Lemma~\ref{lem:regression} once more,
\[\norm{\mmu \hmv_{:j} - \mmh_{:j}}_2^2 \le 3\norm{\Brack{\mn + \mn'}_{:j}}_2^2 + \frac{96}{p\beta^2}\norm{\Brack{\mn + \mn'}_{S_j, j}}_2^2 + \frac{\Delta^2}{n}.\]
The conclusion follows by summing over all columns and using \eqref{eq:sampled_noise} which we conditioned on.
\end{proof}

\subsubsection{Geometric aggregation}

In this section, we give an aggregation technique for boosting the constant error guarantees of earlier sections. We begin with an approximation algorithm for the distance between low-rank matrices.

\begin{algorithm2e}[ht!]\label{alg:dist}
\caption{$\LowRankDist(\mmu, \mv, \mw, \mz, \delta)$}
\DontPrintSemicolon
\codeInput $\mmu, \mw \in \R^{m \times r}$, $\mv, \mz \in \R^{n \times r}$, $\delta \in (0, 1)$\;
$d \gets \lceil 1000\log \frac m \delta \rceil$ \;
Sample $\mq \in \R^{d \times m}$ with independently random unit vector rows in $\R^m$\;
$\tmd \gets \frac 1 {\sqrt d}(\mq \mmu \mv^\top - \mq \mw \mz^\top)$ \;
\Return{$\normsf{\tmd}$}\;
\end{algorithm2e}

\begin{lemma}\label{lem:distgood}
Let $\mm, \mm' \in \R^{m \times n}$ be given as rank-$r$ factorizations $\mmu \mv^\top$, $\mw \mz^\top$ respectively. For any $\delta \in (0, 1)$, $\LowRankDist(\mmu, \mv, \mw, \mz)$ returns a value $V$ such that with probability $\ge 1 - \delta$, 
\[|V - \normf{\mm - \mm'}| \le 0.1\normf{\mm - \mm'}.\]
The runtime of the algorithm is $O((m + n)r\log\frac m \delta)$.
\end{lemma}
\begin{proof}
First, letting $\md \defeq \mmu \mv^\top - \mw \mz^\top$, standard guarantees on Johnson-Lindenstrauss sketches \cite{DasguptaG03} guarantee that with probability at least $1 - \delta$,
\[\Abs{\normf{\md}^2 - \normsf{\tmd}^2} \le 0.1\normf{\md}^2 \implies \Abs{\normf{\md} - \normsf{\tmd}} \le 0.1\normf{\md}, \]
since multiplying by $d^{-\half} \mq$ preserves all row norms of $\md$ up to a $0.1$ factor with this probability. Finally, we can explicitly compute $\tmd$ and return its Frobenius norm in time $O((m + n)rd)$.
\end{proof}

Leveraging Lemma~\ref{lem:distgood}, we give our approximation-tolerant geometric aggregation technique. The algorithm is identical to Algorithm 4 of \cite{KelnerLLST22} other than our use of approximate distance computations, but we provide an analysis of this modification here for completeness.

\begin{algorithm2e}[ht!]\label{alg:agg}
\caption{$\Agg(\{\mm_i\}_{i \in [k]}, \Delta, \delta)$}
\DontPrintSemicolon
\codeInput $\{\mm_i\}_{i \in [k]} \subset \R^{m \times n}$ each given as rank-$r$ factorizations $\{\mmu_i \mv_i^\top\}_{i \in [k]}$, $\Delta \ge 0$ such that $\normf{\mm_i - \mms} \le \Delta$ for an unknown $\mms \in \R^{m \times n}$ and at least $0.51k$ of the $i \in [k]$, $\delta \in (0, 1)$\;
\lFor{$(i, j) \in [k] \times [k]$}{$d_{ij} \gets \LowRankDist(\mmu_i, \mv_i, \mmu_j, \mv_j, \frac{\delta}{k^2})$}
\For{$i \in [k]$}{
\lIf{$d_{ij} \le 2.2\Delta \textup{ for at least } 0.51k \textup{ distinct } j \in [k]$}{\label{line:ifmanyclose}
\Return{$i$}
}
}
\end{algorithm2e}

\begin{lemma}\label{lem:agg}
Under the input assumptions of $\Agg$, with probability $\ge 1 - \delta$, an index $i$ is returned in time $O((m + n)rk^2\log \frac{mk}{\delta})$ satisfying
\begin{equation}\label{eq:outputgood}\normf{\mm_i - \mms} \le 4\Delta.\end{equation}
\end{lemma}
\begin{proof}
We condition on all calls to $\LowRankDist$ returning a pairwise distance up to $0.1$ error, giving the failure probability and runtime via an application of Lemma~\ref{lem:distgood}. To prove \eqref{eq:outputgood}, let 
\[T \defeq \Brace{i \in [k] \mid \normf{\mm_i - \mms} \le \Delta}.\]
Note that any $i \in T$ passes the check on Line~\ref{line:ifmanyclose} by the triangle inequality, so the algorithm will return. Further, any index $i \in [k]$ with $\normf{\mm_i - \mms} \ge 4\Delta$ will fail the check on Line~\ref{line:ifmanyclose} by the triangle inequality, since its (approximate) distance to any $i \in T$ is too large.
\end{proof}

\subsection{Proof of Proposition~\ref{prop:fix}}\label{ssec:fix}

We now put all the pieces together in Algorithm~\ref{alg:fix}, and prove Proposition~\ref{prop:fix}.

\begin{algorithm2e}[ht!]\label{alg:fix}
\caption{$\Fix(\orzo(\mmh), \mm, \rs, \sigma, S, T, \Delta, \alpha, \beta, \mu, \delta)$}
\DontPrintSemicolon
\codeInput $\orzo(\mmh)$ for $\mmh = \mms + \mn \in \R^{m \times n}$ where $\mms$ is rank-$\rs$, $\normop{\mms} \le \sigma$ and $\normf{\mn} \le \frac \Delta {20}$, $\mm \in \R^{m \times n}$ given as a rank-$r$ factorization, $S \subseteq [m]$, $T \subseteq [n]$, $\Delta, \mu \ge 0$, $\alpha, \beta, \delta \in (0, 1)$ \;
$p \gets \frac{4.8 \cdot 10^5 \mu r \log(m)\log(\frac{600m}{\delta})}{\alpha \beta^2 n}$\;
$(S', T') \gets \Sparsify(\orzo(\mmh_{S, T}), \mm_{S, T}, \frac{120\sqrt{15\mu r \log(m)}}{\alpha \beta n}\Delta, 1.05\Delta, \frac{\alpha}{1800\log(m)}, \frac \alpha 9, p, \frac \delta 6)$\;
$K \gets \lceil10\log \frac 6 \delta \rceil$\;
\For{$k \in [K]$}{\label{line:first_loop_start}
$q \gets \frac{750\mu\rs}{\beta^2 n}\log(n)$, $q' \gets \frac{750\rs}{\alpha\beta^2 n}\log(n)$\;
$B_k \gets \Representative(\mm_{S', T'}, \frac{14 \log(m)}{\beta \sqrt{n}} \cdot \Delta, q)$ \;
$(\mmu_k, \mv_k) \gets \Complete(\oracle_{q'}(\mmh_{S':}), \mm_{S',B}, \rs, B, \frac \Delta {20}, \frac{88000\mu\rs\log(\rs)}{\beta^3\sqrt{n}} \cdot \Delta, \sigma, \frac{2\alpha}{3}, \beta)$\;
}\label{line:first_loop_end}
$k^\star \gets \Agg(\{\mmu_k\mv_k^\top\}_{k \in [K]}, \frac{10^8\rs\sqrt{\log(\rs)}}{\beta^5} \Delta, \frac \delta 6)$\;
$(\mmu, \mv) \gets (\mmu_{k^\star}, \mv_{k^\star})$\;\label{line:learned_uv}
\For{$k \in [K]$}{\label{line:first_loop_start}
$(\mv_k, \mmu_k) \gets \Complete(\oracle_{q'}(\mmh^\top), \mv\mmu^\top, \rs, S', \frac \Delta {20}, \frac{4\cdot10^8\rs\sqrt{\log(\rs)}}{\beta^5}\Delta, \sigma, \frac{2\alpha}{3}, \beta)$
}
$k^\star \gets \Agg(\{\mmu_k \mv_k^\top\}_{k \in [K]}, \frac{10^{10}\rs\sqrt{\rs\log(\rs)}}{\beta^8}\Delta, \frac \delta 6)$\;
\Return{$(\mmu_{k^\star}, \mv_{k^\star})$}
\end{algorithm2e}

\restatefix*
\begin{proof}
First, by applying Lemma~\ref{lem:postprocess} with $\gdrop = \frac \alpha 9$ and $\Delta \gets 1.05\Delta$ (to account for the error due to $\mn$), with probability $\ge 1 - \frac \delta 6$ we have that $|S'| \ge m - \frac {\alpha n} 3$ and $|T'| \ge (1 - \frac{\alpha } 3) n$, and that $\mm_{S', T'}$ and $\mmh_{S', T'}$ are $2.2\Delta$-close away from an $s$-RCS matrix (accounting for $\mn$ again), for
\[s \defeq \frac{\alpha \beta^2 n}{15 \cdot 10^4 \mu r\log m}.\]
Condition on this event for the remainder of the proof. Next, consider one run $k \in [K]$ of the loop from Line~\ref{line:first_loop_start} to Line~\ref{line:first_loop_end}. It is straightforward to check that for $p \gets \frac{40\mu\rs\log(\rs)}{\beta^2 n}$ and $\phi \gets \frac{14\log(m)}{\beta\sqrt{n}}$, the preconditions of Lemma~\ref{lem:column-tester} are met because we have $2.2\Delta$-closeness between $\mm_{S', T'}$ and $\mms_{S', T'}$ away from an $s$-RCS matrix, and $\mms_{S', T'}$ has $(\frac {2\alpha} 3, \beta, \mu)$-standard row and column spans. Therefore, with probability $\ge 0.9$, $B_k$ is $(\tDelta, \gamma)$-representative with respect to $\mm_{S', T'}$ and $\mms_{S', T'}$ for
\[\tDelta \defeq \frac{88000\mu\rs\log(\rs)}{\beta^3\sqrt{n}} \Delta,\; \gamma \defeq \sqrt{\frac{q\beta^2}{2}}.\]
Under this event, Lemma~\ref{lem:regression_error} shows $\Complete$ returns a rank-$r'$ factorization $(\mmu_k, \mv_k)$ satisfying
\[\normf{\mmu_k \mv_k^\top - \mms} \le \frac{10^8 \rs \sqrt{\log(\rs)}}{\beta^5} \Delta,\]
with probability $\ge 0.9$, and guarantees $r' \le 2\rs$. Therefore this occurs with probability $\ge 0.8$ for each independent run $k \in [K]$. A Chernoff bound shows the preconditions of $\Agg$ are met with probability $\ge 1 - \frac \delta 6$, and then with probability $\ge 1 - \frac \delta 6$, Lemma~\ref{lem:agg} implies that on Line~\ref{line:learned_uv},
\[\normf{\mmu \mv^\top - \mms_{S':}} \le \Delta' \defeq \frac{4\cdot 10^8\rs\sqrt{\log(\rs)}}{\beta^5}\Delta.\]
Next, note that $S'$ is a $(\Delta', \beta)$-representative subset with respect to any extension of $\mv\mmu^\top$ to $\R^{n \times m}$ and $(\mms)^\top$, by subspace regularity and Lemma~\ref{lem:regular-wc}. An analogous argument to the above shows that with probability $\ge 1 - \frac {\delta} 3$, applying $\Complete$ and $\Agg$ with the given parameters yields \eqref{eq:fix_bound}. Union bounding over all these events, we have a failure probability of $1 - \frac {5\delta} 6$. We condition on one last event with failure probability $\frac \delta 6$ via standard Chernoff bounds: that the total number of observed entries in $\Sparsify$, and the total number of sampled rows and columns in calls to $\Representative$ and $\Complete$, are within constant factors of their expectations.

Regarding the choice of $p$ in the statement, note that the only subroutines which require observations are $\Sparsify$ and $\Complete$, and our bound then follows from our parameter choices and Lemma~\ref{lem:multiple_obs} (the dominant term is the $O(\log(\frac m {\beta}))$ observation calls used by $\Sparsify$). Finally, we discuss runtime. There are four components to bound: $\Sparsify$, $\Representative$, $\Complete$, and $\Agg$. The runtime bottleneck of $\Sparsify$ is computing $O(pmn)$ observations $O(\log \frac m \beta)$ times, where each observation takes time $O(r)$ to compute by our low-rank factorization. The runtime of $\Representative$ is dominated by $O(\log(m))$ calls to $\Test$, and each call solves a regression problem in a $O(m) \times O(nq)$ matrix, which is within the required budget. The cost of $\Complete$ is dominated by running $\AGD$ for $O(\frac 1 \beta \log(\frac{m(\Delta + \sigma)}{\beta\Delta}))$ iterations for each column, and the total number of nonzero entries among all regression matrices is $O(mn\rs q')$, assuming $r' \le 2\rs$. We remark that in the second application of $\Complete$, we need to take an SVD of an $n \times \Theta(m)$ matrix, but its row space is given as an orthonormal basis, so we may apply Lemma~\ref{lem:lowrank_svd} to perform this efficiently. Finally, by an application of Lemma~\ref{lem:agg}, the calls to $\Agg$ do not dominate the runtime.
\end{proof}

\begin{lemma}\label{lem:lowrank_svd}
Let $\mm = \mmu \mv^\top \in \R^{m \times n}$ be given as a rank-$r$ factorization and suppose $\mmu \in \R^{m \times r}$ has orthonormal columns and $\mv \in \R^{n \times r}$. We can compute an SVD of $\mm$ in time $O((m+n)r^2)$.
\end{lemma}
\begin{proof}
Let an SVD be $\mz \msig \mw^\top$. The right singular vectors $\mw$ are an $n \times r$ matrix with orthonormal columns corresponding to the nonzero eigenvalues of $\mv\mv^\top$, and we can compute these in the given time by forming $\mv^\top \mv$, performing eigendecomposition, and multiplying by $\mv$. This also yields the diagonal matrix $\msig$. We can then directly compute $\mz = \mmu \mv^\top \mw \msig^{-1}$ within the allotted time.
\end{proof}
\section{Matrix completion algorithms}\label{sec:algos}

\subsection{Estimating the operator norm}

Our algorithms in Section~\ref{sec:fixing}, as well as computation of an initial distance bound, require an estimate on $\normsop{\mms}$. We give a simple algorithm for performing this estimation under a boundedness assumption on the noise. We then justify that this noise boundedness assumption is without loss of generality, up to a small overhead in our recovery guarantee.

\begin{algorithm2e}[ht!]\label{alg:opnorm}
\caption{$\OpNorm(\orzo(\mmh), p, \delta)$}
\DontPrintSemicolon
\codeInput $\orzo(\mmh)$, $p, \delta \in (0, 1)$\;
$T \gets \lceil 20\log \frac 1 \delta \rceil$ \;
\For{$t \in [T]$}{
$s_t \gets \sqrt{\frac{32}{p\beta^2}}\normf{\oracle_p(\mmh)}$ \;
}
\Return{$\textup{median}(\{s_t\}_{t \in [T]})$}
\end{algorithm2e}

\begin{lemma}\label{lem:approx_op_norm}
Assume $\mmh = \mms + \mn$ where $\mms \in \R^{m \times n}$ is rank-$\rs$ with $(\alpha,\beta,\mu)$-standard row and column spans, and $m \ge n$.
If $\normf{\mn} \le \frac \beta {10}\normf{\mms}$ and $p \ge \frac{30\mu \rs}{\beta^2 n}\log(n)$, Algorithm~\ref{alg:opnorm} returns a value $V$ such that with probability $\ge 1 - \delta$, $\normop{\mms} \le V \le 2\sqrt n\normop{\mms}$.
\end{lemma}
\begin{proof}
Consider one independent run of the loop in Algorithm~\ref{alg:opnorm}, and let $\Omega$ be the observed entries. With probability at least $\frac 2 3$, by Markov's inequality we have
\[\normf{\mn_{\Omega}}^2 \le \frac {p\beta^2} {100} \normf{\mms}^2,\]
where we used the assumption on $\normf{\mn}$. Further, let $S_j \subseteq [m]$ be the observed entries in column $j$ for all $j \in [n]$, and let $\mmus \msigs \mvs^\top$ be an SVD of $\mms$. With probability at least $\frac 1 {15n}$ for each $j \in [n]$, by an analogous argument to the lower bound in \eqref{eq:wc_random_sample} (since adding outer products of rows can only increase the smallest eigenvalue), we have that $\norm{[\mmus]_{S_j:}v}_2^2 \ge \frac{p\beta^2}{8}\norm{v}_2^2$ for all $v \in \R^{\rs}$. Therefore, by a union bound on this event over all $j \in [n]$ we have with probability at least $\frac {14} {15}$,
\begin{align*}
\normf{\mms_\Omega}^2 = \sum_{j \in [n]} \norm{[\mms]_{S_j, j}}_2^2 \ge \frac{p\beta^2}{8} \sum_{j \in [n]} \norm{\mms_{:j}}_2^2 = \frac{p\beta^2}{8}\normf{\mms}^2.
\end{align*}
Combining the above two displays and taking a union bound implies that in each independent run, with probability at least $\frac 3 5$, we have
\begin{align*}
\frac{32}{p\beta^2}\normf{\mmh_{\Omega}}^2 \ge \frac{32}{p\beta^2}\Par{\half \normf{\mms_\Omega}^2 - 2\normf{\mn}^2} \ge \normf{\mms}^2,
\end{align*}
where we applied $(a + b)^2 \ge \half a^2 - b^2$ entrywise to $\mmh = \mms + \mn$. Applying a Chernoff bound then implies the median estimate over the runs satisfies the above display with probability $\ge 1 - \frac \delta 2$, which gives the upper bound on $\normop{\mms} \le \normf{\mms} \le V$. For the lower bound,
\begin{align*}
\normf{\mmh_\Omega}^2 \le 2\normf{\mms_\Omega}^2 + 2\normf{\mn_\Omega}^2 \le 3\normf{\mms}^2 \le 3\rs\normop{\mms}^2,
\end{align*}
for each independent run with probability at least $\frac 3 5$ by conditioning on the same event on $\mn$ as before. A similar Chernoff bound and $\frac{96\rs}{p\beta^2} \le 4n$ then yields the upper bound on $V$.
\end{proof}

\begin{remark}\label{remark:boundN}
In the regime $\normf{\mn} \ge  \frac \beta {10} \normf{\mms}$, the revealed matrix $\mmh = \mms + \mn$ can equivalently be written as $\mmh = \mzero_{m \times n} + (\mms + \mn)$, where we treat $\mzero_{m \times n}$ as the target low-rank matrix and $(\mms + \mn)$ as the noise. This only increases the target noise level by a $\frac{11}{\beta}$ factor.
\end{remark}

\subsection{Main result}

We are now ready to state our main meta-result for matrix completion.

\begin{algorithm2e}[ht!]\label{alg:mc-main}
\caption{$\MC(\orzo(\mmh), \rs, \alpha, \beta, \mu, \Delta, \delta)$}
\DontPrintSemicolon
\codeInput $\orzo(\mmh)$, $\rs \in \N$, $\mu, \Delta \ge 0$, $\alpha, \beta, \delta \in (0, 1)$\;
$\Delta \gets \frac {11\Delta} \beta$\;
$\sigma \gets \OpNorm(\mmh, \frac{30\mu\rs}{\beta^2 n}\log(n), \frac \delta 4)$ \;
$\ell \gets \exp(\sqrt{\log (\rs\beta^{-1}}))$ \;
$K \gets \frac 1 {\log \ell} \cdot \log(2\Cfix \rs\sqrt{\rs\log(\rs)}\beta^{-8})$\;
$\tDelta \gets \sqrt{\rs}\sigma$\;
$(\mmu, \mv) \gets (\mzero_{m \times 0},\mzero_{n \times 0})$\;
$k \gets 0$ \;
$(S, T) \gets ([m], [n])$\;
$N \gets K\log_2(\frac{\tDelta}{20\ell\Delta})$\;
$\gadd \gets \frac{\alpha}{9 \cdot 10^5\log(\frac{m}{\alpha\beta}) \ell^2K^2}$\;
\While{$\tDelta \ge 20\ell \Delta$}{\label{line:while_start}
$(\mmu, \mv, S, T) \gets \Descent(\orzo(\mmh_{S, T}), [\mmu \mv^\top]_{S, T}, \rs, \tDelta, \gadd k, \gadd, \frac{\delta}{4N}, \ell)$\;
$\tDelta \gets \frac \Delta \ell$\;
$k \gets k + 1$\;
\If{$k = K$}{
$(\mmu, \mv) \gets \Fix(\orzo(\mmh), \mmu\mv^\top, \rs, \sigma, S, T, \tDelta, \alpha, \beta, \mu, \frac \delta {4N})$\;
$\tDelta \gets \Cfix \rs \sqrt{\rs\log(\rs)}\tDelta$\;
$(S, T) \gets ([m], [n])$\;
$k \gets 0$\;
}
}\label{line:while_end}
$(\mmu, \mv^\top) \gets$ top $\rs$ components of an SVD of $\Fix(\orzo(\mmh), \mmu\mv^\top, \rs, \sigma, S, T, \tDelta, \alpha, \beta, \mu, \frac \delta {4})$ sorted by the corresponding singular value \;
\Return{$(\mmu, \mv)$}
\end{algorithm2e}

\begin{theorem}\label{thm:main}
Let $\mms \in \R^{m \times n}$ be rank-$\rs$ with $(\alpha, \beta, \mu)$-row and column spans, $m \ge n$, $\delta \in (0, 1)$, and let $\mmh = \mms + \mn$ for $\normf{\mn} \le \Delta$. Algorithm~\ref{alg:mc-main} returns $\mmu \in \R^{m \times \rs}$ and $\mv \in \R^{n \times \rs}$ satisfying
$\normf{\mmu \mv^\top - \mms} \le \frac{(\rs)^{1.5 + o(1)}}{\beta^9} \Delta$,
with probability $\ge 1 - \delta$. Algorithm~\ref{alg:mc-main} uses 
\[O\Par{\frac{m(\rs)^{2 + o(1)} \mu^2}{\alpha \beta^{4 + o(1)}} \cdot \Par{\log^6\Par{\frac m {\alpha\beta\delta}}\log\Par{\frac{m\normop{\mms}}{\Delta\beta\delta}} + \log^{2.5}\Par{\frac m {\alpha\beta\delta}}\log^2\Par{\frac{m\normop{\mms}}{\Delta\beta\delta}} } }\]
time and one call to $\oracle_p(\mmh)$ where for a sufficiently large constant,
\[p = O\Par{\frac{(\rs)^{1 + o(1)} \mu }{\alpha \beta^{2 + o(1)} n} \cdot \log^{6}\Par{\frac m {\alpha\beta\delta}}\log\Par{\frac{n\normop{\mms}}{\Delta}}}.\]
\end{theorem}
\begin{proof}
By Remark~\ref{remark:boundN} and the guarantees of Lemma~\ref{lem:approx_op_norm}, our estimate $\sigma$ is an upper bound on $\normop{\mms}$ with probability at least $1 - \frac \delta 4$; we condition on this for the remainder of the proof. This also implies that our initial estimate $\tDelta$ is a valid overestimate of $\normf{\mmu \mv^\top - \mms} \le \sqrt{\rs} \normop{\mms}$ at the beginning of the algorithm. We next claim that throughout the algorithm, $[\mmu \mv^\top]_{S, T}$ and $\mms_{S, T}$ are $\tDelta$-close on a $\gadd k$-submatrix. This invariant is preserved every time we call $\Descent$ (assuming it succeeds), by Proposition~\ref{prop:iterative-step}. Further, our parameter settings imply the preconditions of Proposition~\ref{prop:fix} are met whenever it is called: it is straightforward to check that the $\gdrop$ parameter in Proposition~\ref{prop:iterative-step} is bounded by $\frac \alpha {9K}$, so that after $K$ steps, at most an $\frac \alpha 9$ fraction of rows and columns are dropped, and the submatrix parameter is at most $\gadd K \le \frac{\alpha}{1800\log(m)}$. Hence, every time we call $\Fix$ (assuming it succeeds) the invariant is also preserved, by the guarantees of Proposition~\ref{prop:fix}.

The above argument also shows that every time the loop in Lines~\ref{line:while_start} to~\ref{line:while_end} is executed, $\tDelta$ is decreased by a factor of $\ell^K \cdot (\Cfix\rs\sqrt{\rs\log(\rs)}) = 2$, by combining the guarantees of Proposition~\ref{prop:iterative-step} ($K$ times) and Proposition~\ref{prop:fix} (once). This implies that the number of times the loop is executed is at most $N$. By union bounding over all $N$ calls to $\Descent$ and $\Fix$, the last call to $\Fix$, and the first call to $\OpNorm$, this gives the failure probability; we condition on all of these calls succeeding for the remainder of the proof. When the algorithm exits the loop and before $\Fix$ is called for the last time, the closeness parameter (on a submatrix) is bounded by $20\ell\Delta$, so the distance bound follows from Proposition~\ref{prop:fix} and since we increased $\Delta$ by a $\frac 1 \beta$ factor at the start of the algorithm. Finally, we note that because the top-$\rs$ truncation of the output's SVD minimizes the projection to rank-$\rs$ matrices by Frobenius norm, the distance to $\mms$ (which is rank-$\rs$) can at most double.

Further, note that throughout the algorithm, we can inductively apply Proposition~\ref{prop:iterative-step} to maintain that the rank $r$ of our iterate is bounded by $3^{k + 1}\rs = (\rs)^{1 + o(1)} \beta^{-o(1)}$, since the potential function $r + \rs$ at most triples each iteration, and whenever $k$ is reset to $0$, Proposition~\ref{prop:fix} guarantees that $r \le 2\rs$. The bounds on the runtime and $p$ then follow by combining Propositions~\ref{prop:iterative-step} and~\ref{prop:fix} (at most $N + 1$ times) with Lemma~\ref{lem:approx_op_norm}, where we apply Lemma~\ref{lem:multiple_obs} to aggregate the observation probabilities. To handle the runtime of the final SVD and truncation, it suffices to use Lemma~\ref{lem:lowrank_svd}.
\end{proof}

By combining Theorem~\ref{thm:main} with Facts~\ref{fact:mu_alpha} and~\ref{fact:incoherent-standard}, we then obtain the following results.

\begin{corollary}\label{cor:main_regular}
Let $\mms \in \R^{m \times n}$ be rank-$\rs$ with $(\Omega(1), \Omega(1))$-regular row and column spans, $m \ge n$, $\delta \in (0, 1)$, and let $\mmh = \mms + \mn$ for $\normf{\mn} \le \Delta$. Algorithm~\ref{alg:mc-main} returns $\mmu \in \R^{m \times \rs}$ and $\mv \in \R^{n \times \rs}$ satisfying $\normf{\mmu \mv^\top - \mms} \le (\rs)^{1.5 + o(1)} \Delta$,
with probability $\ge 1 - \delta$. Algorithm~\ref{alg:mc-main} uses 
\[O\Par{m(\rs)^{2 + o(1)}  \cdot \Par{\log^6\Par{\frac m {\delta}}\log\Par{\frac{m\normop{\mms}}{\Delta\delta}} + \log^{2.5}\Par{\frac m {\delta}}\log^2\Par{\frac{m\normop{\mms}}{\Delta\delta}} } }\]
time and one call to $\oracle_p(\mmh)$ where for a sufficiently large constant,
\[p = O\Par{\frac{(\rs)^{1 + o(1)} }{n} \cdot \log^{6}\Par{\frac m {\delta}}\log\Par{\frac{n\normop{\mms}}{\Delta}}}.\]
\end{corollary}

\begin{corollary}\label{cor:main_incoherent}
Let $\mms \in \R^{m \times n}$ be rank-$\rs$ with $\mu$-incoherent row and column spans, $m \ge n$, $\delta \in (0, 1)$, and let $\mmh = \mms + \mn$ for $\normf{\mn} \le \Delta$. Algorithm~\ref{alg:mc-main} returns $\mmu \in \R^{m \times \rs}$ and $\mv \in \R^{n \times \rs}$ satisfying
$\normf{\mmu \mv^\top - \mms} \le (\rs)^{1.5 + o(1)} \Delta$,
with probability $\ge 1 - \delta$. Algorithm~\ref{alg:mc-main} uses 
\[O\Par{m(\rs)^{3 + o(1)} \mu^3 \cdot \Par{\log^6\Par{\frac {m} {\delta}}\log\Par{\frac{m\normop{\mms}}{\Delta\delta}} + \log^{2.5}\Par{\frac {m}{\delta}}\log^2\Par{\frac{m\normop{\mms}}{\Delta\delta}} } }\]
time and one call to $\oracle_p(\mmh)$ where for a sufficiently large constant,
\[p = O\Par{\frac{(\rs)^{2 + o(1)} \mu^2 }{n} \cdot \log^{6}\Par{\frac m {\delta}}\log\Par{\frac{n\normop{\mms}}{\Delta}}}.\]
\end{corollary}

\section*{Acknowledgements}
We thank Yeshwanth Cherapanamjeri for communications on the prior work \cite{cherapanamjeri2017nearly}.

\bibliographystyle{alpha}
\bibliography{ref}
\newpage
\begin{appendix}
\section{Regularity of random subspaces}\label{app:regular-subspace}

In this section, we prove that uniformly random subspaces of $\R^d$ of dimension $r$ are $(\Omega(1), \Omega(1))$-regular with exponentially small failure probability, when $\frac d r$ is at least a sufficiently large constant. This latter condition is not typically restrictive, as the regime of interest in matrix completion is where $r = o(\min(m, n))$ (otherwise, it is information-theoretically necessary to reveal at least a constant fraction of the matrix, limiting the runtime gains of matrix completion algorithms). Our main helper tool is the following standard concentration bound on the spectra of Wishart matrices.

\begin{lemma}\label{lem:wishart_cond}
Let $\mg \in \R^{d \times r}$ have independent entries $\sim \Nor(0, 1)$, and assume $\frac d r$ is sufficiently large. For a universal constant $C$, $\kappa(\mg^\top \mg) \le 3$ with probability $\ge 1 - \exp(-Cd)$ (where recall $\kappa(\ma)$ denotes the condition number of a matrix $\ma$).
\end{lemma}
\begin{proof}
For shorthand let $\mk \defeq \mg^\top \mg \in \R^{r \times r}$, where $\E \mk = d \id_r$. Letting $N$ be a maximal $0.1$-net of the unit ball in $\R^r$, Lemma 1.18 of \cite{RigolletH17} shows $|N| \le \exp(4r)$. By Exercise 4.3.3 of \cite{Vershynin16}, 
\begin{align*}
\normop{\mk - \id_r} \le 1.25 \max_{v \in N} \Abs{v^\top\Par{\mk - \id_r} v},
\end{align*}
so it suffices to prove that with the desired probability, we simultaneously have $|v^\top \mk v - d| \le 0.4d$ for all $v \in N$. For any $v \in N$, $v^\top \mk v$ is a chi-squared random variable with $d$ degrees of freedom, so 
\[\Pr[|v^\top \mk v - d| > 0.4d] \le \exp(-2Cd)\]
for $C \ge \frac 1 {80}$, by Lemma 1 of \cite{LaurentM00}. The conclusion follows from a union bound for $4r \le Cd$.
\end{proof}

\begin{corollary}\label{cor:usually_reg}
Let $V \subseteq \R^d$ be a uniformly random subspace of dimension $r$, where $\frac d r$ is sufficiently large. For universal constants $\alpha$ and $\gamma$, $V$ is $(\alpha, \frac 1 3)$-regular with probability $\ge 1 - \exp(-\gamma d)$.
\end{corollary}
\begin{proof}
By the characterization in Lemma~\ref{lem:regular-wc} (and following its notation), it suffices to prove that for every $S \subset [d]$ with $|S| = \lceil (1 - \alpha) d \rceil$, we have $\kappa(\sum_{i \in S} b_i b_i^\top) \le 9$, since taking larger $S$ can only improve the condition number. Let $\alpha$ be a sufficiently small constant such that 
\[\binom{d}{d - \lceil (1 - \alpha) d\rceil} \le \exp\Par{\frac{Cd}{3}},\]
which exists following the estimate $\binom{d}{k}\le(\frac{ed}{k})^k$. By rotational symmetry, it suffices to consider
$\mb_V = \mk^{-\half} \mg$, following the notation of Lemma~\ref{lem:regular-wc}. In this case we further have
\[\sum_{i \in S} b_ib_i^\top = \mk^{-\half} \mk_S \mk^{-\half} \text{ for } \mk_S \defeq \mg_{S:}^\top \mg_{S:}.\]
Finally, with probability at least $1 - (\exp(-Cd) + \exp(-\frac{Cd}{3})) \ge 1 - \exp(-\gamma d)$ for an appropriate constant $\gamma$, $\mk$ and $\mk_S$ (for all $|S| = \lceil (1 - \alpha) d\rceil$ simultaneously) satisfy the conclusion of Lemma~\ref{lem:wishart_cond}. Therefore, the claim follows from Lemma~\ref{lem:kappatimes}:
\[\kappa\Par{\mk^{-\half} \mk_S \mk^{-\half}} \le \kappa\Par{\mk^{-1}}\kappa\Par{\mk_S} = \kappa\Par{\mk}\kappa\Par{\mk_S} \le 9. \]
\end{proof}

\begin{lemma}\label{lem:kappatimes}
For any positive definite $\ma, \mb \in \R^{d \times d}$, $\kappa(\ma^{\half}\mb\ma^{\half}) \le \kappa(\ma)\kappa(\mb)$.
\end{lemma}
\begin{proof}
It suffices to take a ratio of the following bounds:
\begin{align*}
\lam_1\Par{\ma^{\half}\mb\ma^{\half}} &= \max_{\norm{u}_2 = 1} u^\top \ma^{\half}\mb\ma^{\half} u \le \lam_1(\ma) \max_{\norm{v}_2 = 1} v^\top \mb v =\lam_1(\ma)\lam_1(\mb), \\
\lam_d\Par{\ma^{\half} \mb\ma^{\half}} &= \min_{\norm{u}_2 = 1} u^\top \ma^{\half}\mb\ma^{\half} u \ge \lam_d(\ma) \min_{\norm{v}_2 = 1} v^\top \mb v = \lam_d(\ma)\lam_d(\mb).
\end{align*}
\end{proof}
\section{One-sided matrix discrepancy bound}\label{app:deferred-fixing}

In this section, we prove Proposition~\ref{prop:exists-lower-bound}, a one-sided matrix discrepancy bound. Our proof follows from a straightforward application
of the resolution of the Kadison-Singer conjecture from Marcus, Spielman and Srivastava. In particular, we use the following result restated from \cite{MarcusSS15}.
\begin{proposition}[Specialization of Corollary 1.5, \cite{MarcusSS15}]
\label{prop:ks} For any $t \in \N$ and $\{u_i\}_{i \in [m]} \subset \R^{r}$
such that $\sum_{i\in[m]}u_{i}u_{i}^{\top}=\id_r$ and $\norm{u_{i}}_{2}^2\leq\delta$
for all $i\in[m]$, there is a partition $\{S_j\}_{j \in [t]}$
of $[m]$ with
\[
\normop{\sum_{i\in S_{j}}u_{i}u_{i}^{\top}}\leq\left(\frac{1}{\sqrt{t}}+\sqrt{\delta}\right)^{2}=\frac{1}{t}\left(1+\sqrt{\delta t}\right)^{2}\text{ for all } j \in [t].
\]
\end{proposition}

As a corollary of Proposition~\ref{prop:ks} we have the following result on splitting
an approximation of a multiple of the identity into two pieces; we will later apply this procedure recursively. 
\begin{corollary}
\label{cor:cor} For any $\{v_i\}_{i \in [m]} \subset \R^{r}$ and $\lam > 0$ such that
$\sum_{i\in[m]}v_{i}v_{i}^{\top}\approx_{\epsilon}\lambda\id_r$ and $\norm{v_{i}}_{2}^{2}\leq\delta$
for all $i\in[m]$, and $\epsilon\in(0, \frac{1}{4})$, $\delta \in (0, \frac{\lambda}{100})$, there exists a partition $\{S_{1},S_{2}\}$
of $[m]$ such that for all $j \in [2]$,
\[
\sum_{i\in S_{j}}u_{i}u_{i}^{\top}\approx_{\epsilon+5\sqrt{\delta/ \lam}}\frac{\lambda}{2}\id_r.
\]
\end{corollary}

\begin{proof}
Let $\mm\defeq\sum_{i\in[m]}v_{i}v_{i}^{\top}$ and let $u_{i}=\mm^{-\half}v_{i}$.
Note that 
\[
\sum_{i\in[m]}u_{i}u_{i}^{\top}=\mm^{-\half}\left(\sum_{i\in[m]}v_{i}v_{i}^{\top}\right)\mm^{-\half}=\id_r,
\]
and 
\[
\norm{u_{i}}_{2}^{2}=v_{i}^{\top}\mm^{-1}v_{i}\leq\frac 1 \lam \exp(\epsilon)\norm{v_i}_2^2\leq \frac \delta \lam \exp(\epsilon).
\]
Applying Proposition~\ref{prop:ks} to $\{u_{i}\}_{i \in [m]}$ with
$t=2$ yields a partition $\{S_{1},S_{2}\}$ of
$[d]$ such that
\begin{align}
\normop{\sum_{i\in S_{j}}u_{i}u_{i}^{\top}} & \leq\frac{1}{2}\left(1+\sqrt{\frac{2\delta}{\lam}\exp(\epsilon)}\right)^{2}=\frac{1}{2}\left(1+2\sqrt{\frac{2\delta}{\lam}\exp(\epsilon)}+\frac{2\delta}{\lam}\exp(\epsilon)\right)\nonumber \\
 & \leq\frac{1}{2}\left(1+\sqrt{\frac{12\delta}{\lam}}\right)\leq\frac{1}{2}\exp\left(\sqrt{\frac{12\delta}{\lam}}\right), \text{ for all } j \in [2],\label{eq:eval_upper}
\end{align}
where we used that $2\sqrt{2\exp(\frac 1 4)}+2\exp(\frac 1 4)\cdot \frac 1 {10}\leq\sqrt{12}$.
Consequently, for all $x\in\R^{r}$ we have 
\begin{align*}
x^{\top}\left(\sum_{i\in S_{1}}u_{i}u_{i}^{\top}\right)x & =\norm x_{2}^{2}-x^{\top}\left(\sum_{j\in S_{2}}u_{j}u_{j}^{\top}\right)x\geq\norm x_{2}^{2}\left(1-\normop{\sum_{j\in S_{2}}u_{j}u_{j}^{\top}}\right)\\
 & \geq\norm x_{2}^{2}\left(1-\frac{1}{2}\left(1+\sqrt{\frac{12\delta}{\lam}}\right)\right)=\norm x_{2}^{2}\cdot\frac{1}{2}\left(1-\sqrt{\frac{12\delta}{\lam}}\right).
\end{align*}
Using $\sqrt{12\delta/\lambda}\leq\frac{1}{2}$, $1-x\geq\exp(-x-x^{2})$
for all $x\in[0,\frac{1}{2}]$, and $\sqrt{12}+1.2\leq 5$
we have that
\[
1-\sqrt{\frac{12\delta}{\lam}}\geq\exp\Par{-\sqrt{\frac{\delta}{\lam}}\Par{\sqrt{12} + 12\sqrt{\frac{\delta}{\lam}}}}\geq\exp\Par{-4\sqrt{\frac \delta \lam}}.
\]
Combining with \eqref{eq:eval_upper} then yields $\sum_{i\in S_{1}}u_{i}u_{i}^{\top}\approx_{5\sqrt{\delta/\lambda}}\frac{1}{2}\id_r$.
Since $u_{i}=\mm^{-\half}v_{i}$ and $\mm\approx_{\epsilon}\id_r$ the
result follows for $S_1$, and the result for $S_2$ is symmetric.
\end{proof}
Applying Corollary~\ref{cor:cor} repeatedly then yields the following result on
splitting a decomposition of the identity into smaller pieces, inspired by procedures described in \cite{FriezeM99, Srivastava13}. 
\begin{corollary}
\label{cor:split} For any $k \in \N$ and $\{u_i\}_{i \in [m]}\in\R^{r}$ such that $\sum_{i\in[m]}u_{i}u_{i}^{\top}=\id_r$
and $\norm{u_{i}}_{2}^2\leq\delta\leq\frac{1}{1400 \cdot 2^k}$ for all $i\in[m]$, there exists
a partition $\{S_j\}_{j \in [2^k]}$ of $[m]$ such that
\[
\sum_{i\in S_{j}}u_{i}u_{i}^{\top}\approx_{13\sqrt{\delta 2^k}}\frac{1}{2^k}\id_r \text{ for all } j \in [2^\ell].
\]
\end{corollary}

\begin{proof}
We prove the result by induction to show that for all $\ell \in [k]$, under the given assumptions we can find a partition $\{S_j^{(\ell)}\}_{j \in [2^\ell]}$ of $[m]$ such that for all $j \in [2^\ell]$,
\[
\sum_{i\in S_{j}^{(\ell)}}u_{i}u_{i}^{\top}\approx_{\epsilon_{\ell}}\frac{1}{2^\ell}\id_r\text{ where }\epsilon_{\ell}\defeq\sum_{i\in[\ell-1]}5\sqrt{\delta 2^i}.
\]
This suffices to prove the result as
\[
\epsilon_{\ell}=5\sqrt{\delta}\sum_{i\in[\ell-1]}\left(\sqrt{2}\right)^{i}=5\sqrt{\delta}\cdot\left(\frac{\left(\sqrt{2}\right)^{\ell}-1}{\sqrt{2}-1}\right)\leq13\sqrt{\delta2^\ell}.
\]

The base case $\ell=0$ clearly holds as in this case $2^\ell =1$,
$\epsilon_{\ell}=0$, and $\sum_{i\in[d]}u_{i}u_{i}^{\top}=\id_r$. Next,
suppose that the claim holds for some $\ell \in [k - 1]$.
Since $2^{\ell}\leq\frac{1}{2800\delta}$ and $13\sqrt{2800^{-1}}\leq \frac 1 4$
we can apply Corollary~\ref{cor:cor} $2^\ell$ times, where in each application
$\{v_i\}_{i \in [m]}$ is set to the $u_{i}$ in some $S_{j}^{(\ell)}$
and $\lambda$ is set to $\frac{1}{2^\ell}$.
The resulting sets partition $[m]$ into $2^{\ell+1}$ pieces that have
the desired properties. 
\end{proof}
Leveraging Corollary~\ref{cor:split} and a standard, natural splitting argument
we prove our main result.

\restateexistslowerbound*
\begin{proof}
Let $k \in \N$ be such that $\frac 1 {2^{k + 1}} \le \frac \lam 4 \le \frac 1 {2^k}$, and let $\{u_i\}_{i \in [m]}$ be formed by replacing every $b_i$ with $\alpha_i \defeq \lceil\norm{b_i}_2^2 \cdot \frac d r \rceil$ copies of $\frac 1 {\sqrt{\alpha_i}} b_i$. Note that each $\norm{u_i}_2^2 \le \delta \defeq \frac r d$ and
\[m = \sum_{i \in [d]} \alpha_i \le d + \frac d r \sum_{i \in [d]} \norm{b_i}_2^2 = 2d.\]
Now since $\delta \le \frac \lam {5600} \le \frac 1 {1400 \cdot 2^k}$, we can apply Corollary~\ref{cor:split} to the $\{u_i\}_{i \in [m]}$, and let $T \subseteq [m]$ be the smallest cardinality set in the output partition. This set satisfies $|T| \le \frac{2d}{2^k} \le d\lam$, and
\[
\lambda_{\min}\left(\sum_{i\in S}u_{i}u_{i}^{\top}\right)\geq\exp\left(-13\sqrt{\frac{r}{d} \cdot 2^k}\right)\frac{1}{2^k}\geq\frac{1}{2^{k + 1}} \ge \frac \lam {8}.
\]
Finally, letting $S\subseteq[d]$ consist of all indices of a $b_i$ associated with one of the $u_i$ indexed by $T$, we have $\sum_{i\in[S]}b_{i}b_{i}^{\top}\succeq\sum_{i\in T}u_{i}u_{i}^{\top}$
and $|S|\leq|T|$ since $b_ib_i^\top$ is the sum of all associated $u_iu_i^\top$.
\end{proof}
\end{appendix}

\end{document}